\documentclass[sigconf]{acmart}

\usepackage{algorithm}
\usepackage{algorithmicx}
\usepackage{enumitem}
\usepackage [noend]{algpseudocode}
\usepackage{amsmath,amsfonts,amsthm}
\usepackage{arydshln}
\usepackage{multirow}
\usepackage{natbib}
\usepackage{subcaption}
\usepackage{caption}
\usepackage{svg}

\AtBeginDocument{%
  \providecommand\BibTeX{{%
    \normalfont B\kern-0.5em{\scshape i\kern-0.25em b}\kern-0.8em\TeX}}}

\setcopyright{acmlicensed}
\copyrightyear{2018}
\acmYear{2018}
\acmDOI{XXXXXXX.XXXXXXX}

\acmISBN{978-1-4503-XXXX-X/18/06}

\begin{document}


\title{DPHGNN: A Dual Perspective Hypergraph Neural Networks}

\author{Siddhant Saxena\textsuperscript{$1$}, Shounak Ghatak\textsuperscript{$2$}, Raghu Kolla\textsuperscript{$3$}, Debashis Mukherjee\textsuperscript{$3$}, Tanmoy Chakraborty\textsuperscript{$1$}}
\affiliation{
\country{
    {
        \textsuperscript{$1$}IIT Delhi, India; 
        \textsuperscript{$2$}IIIT Delhi, India; \textsuperscript{$3$}Meesho, India
    }
}
}
\email{siddhantsaxenaphy@gmail.com, shounak19109@iiitd.ac.in,}
\email{{raghu.kolla,debashis.mukherjee}@meesho.com, tanchak@iitd.ac.in}

\begin{abstract}
Message passing on hypergraphs has been a standard framework for learning higher-order correlations between hypernodes. Recently-proposed hypergraph neural networks (HGNNs) can be categorized into spatial and spectral methods based on their design choices. 
In this work, we analyze the impact of change in hypergraph topology on the suboptimal performance of HGNNs and propose DPHGNN, a novel dual-perspective HGNN that introduces equivariant operator learning to capture lower-order semantics by inducing topology-aware spatial and spectral inductive biases. DPHGNN employs a unified framework to dynamically fuse lower-order explicit feature representations from the underlying graph into the super-imposed hypergraph structure. We benchmark DPHGNN over eight benchmark hypergraph datasets for the semi-supervised hypernode classification task and obtain superior performance compared to seven state-of-the-art baselines. We also provide a theoretical framework and a synthetic hypergraph isomorphism test to express the power of spatial HGNNs and quantify the expressivity of DPHGNN beyond the Generalized Weisfeiler Leman (1-GWL) test. Finally, DPHGNN was deployed by our partner e-commerce company for the Return-to-Origin (RTO) prediction task, which shows $~7\%$ higher macro F1-Score than the best baseline. 
\end{abstract}

\maketitle

\section{Introduction}

In recent years, the use of spatial and spectral message-passing neural networks (MPNN) \cite{c:01,verma2019heterogeneous} on graph topology has grown exponentially to solve various downstream tasks such as node label classification, link prediction, and graph classification \cite{c:02}. Graph Neural Networks (GNNs) achieve exceptional performance for representation learning on graph-structured data. GNNs improve contextual feature representation of nodes in the graph via layer-wise spatial feature aggregation and message propagation. A wide range of applications involving graph-based semi-supervised node classification, such as molecular property prediction \cite{c:03}, topic modeling of research papers from a scientific citation network \cite{c:04},  user-item recommendations in e-commerce networks \cite{c:05}, have been explored. However, the graph structure limits modeling the interactions to pair-wise node connections and, therefore, is restricted to capturing only the lower-order correlation and relationships between entities.

Hypergraphs provide a flexible mechanism to model higher-order data correlation and complex relationships by allowing hyperedges of two or more hypernodes. Many real-world applications of semi-supervised node classification involve higher-order relation modeling \cite{c:23}, such as academic citation networks consisting of hypernodes as scientific authors and hyperedges as coauthorship relations among authors. \citet{c:25} used hypergraph structure to describe quantum optical experiments. \citet{c:24} improved collaborative filtering, and \citet{c:27} used session-based recommendations using higher-order relations. \citet{c:29} considered semi-dynamic hypergraphs for 3d pose estimation.
Recent studies proposed models for Hypergraph Neural Networks (HGNNs) \cite{c:06}, Hypergraph Convolution and Hypergraph Attention \cite{c:22}. However, despite the flexibility and advances of message-passing neural networks on hypergraphs, there exist some topological and design-related challenges, as mentioned below:

\begin{itemize}[leftmargin=*]
    \item Performance gap between spatial and spectral HGNNs with modeling data into hypergraph topology. 
    \item Effects of over-smoothing and feature collapse in hypergraph topology with resource-constrained setting (e.g., class imbalance, sparse incidence structure,
    a limited set of features, etc.) .
    \item Exploiting the underlying graph structure to explicitly incorporate information from the lower-order structure to the higher-order message-passing framework.
\end{itemize}

\noindent To address these challenges, we propose DPHGNN ({\bf D}ual {\bf P}erspective {\bf H}yper{\bf g}raph {\bf N}eural {\bf N}etwork),  a mixed spectral and spatial learning framework that improves feature representation learning on resource-constrained hypergraph settings. DPHGNN adopts -- (i) a dual-layered feature update mechanism, (ii) a static update layer to provide spectral inductive biases and lower-order relational features to update the static feature matrix of hypernodes, and (iii) a dynamic update layer to fuse the explicitly aggregated features from the underlying graph topology in the hypergraph spatial message propagation. Moreover, with extensive empirical and theoretical analyses, we show that DPHGNN can tackle the above-mentioned challenges and produce improved feature representations for the downstream hypernode classification task.

\begin{table*}[]
\centering
\scalebox{1}{
\begin{tabular}{l|p{7cm}||l|p{6cm}}
\textbf{Notation} & \textbf{Description}  & \textbf{Notation} & \textbf{Description}\\\hline
$G = (V, E)$ & An undirected graph with a set of nodes $V$ and a set of edges $E$ & $HG = (V, \xi)$ & An undirected hypergraph with a set of hypernodes $V$ and  a set of hyperedges $\xi$\\
$A$ & Adjacency Matrix & $H$  & Hypergraph incidence matrix \\ 
$D_{v}$                                  & Node degree matrix                                                                                                                            & $D_{v}$                                                                                      & Hypernode degree matrix                                                                                                                                               \\
$L_{sym}$                                & Symmetric Laplacian Matrix                                                                                                                    & $D_{e}$                                                                                      & Hyperedge degree matrix                                                                                                                                               \\
$G_{c}$, $A_{c}$                             & Clique decomposition graph and adjacency                                                                                                      & $\Delta_{HGNN}$                                                                                    & HGNN Laplacian matrix                                                                                                                                                 \\
$G_{*}$, $A_{*}$                              & Star decomposition graph and adjacency                                                                                                        & $\Delta_{sym}$                                                                                     & Symmetric Laplacian matrix                                                                                                                                            \\
$G_{hyp}$, $A_{hyp}$                          & HyperGCN decomposition graph  and adjacency                                                                                                   & $\Delta_{rw}$                                                                                      & Random-walk Laplacian matrix                                                                                                                                     \\ \hline
\end{tabular}}
\caption{A summary of notations used throughout the paper.}
\label{table1}
\vspace{-5mm}
\end{table*}
    
We summarize our major contributions below\footnote{The source code of DPHGNN is available at \url{https://github.com/mr-siddy/DPHGNN}.}:


\begin{itemize}[leftmargin=*]
    \item DPHGNN introduces a novel message propagation framework that explicitly diffuses lower-order graph features to super-imposed hypergraph structure.
    \item We introduce equivariant operator learning (EOL) over TAA and SIB inductive biases. EOL creates an information-rich feature mixture, rather than cold-start with HG initial features. This also constitutes maximally expressive, k-order equivariant layers in DPHGNN.
    \item We perform extensive empirical analysis on DPHGNN's generalized performance over eight benchmark datasets, and a new isomorphic HG classification task. We formalize strong mathematical characterization on 3-GWL expressivity of DPHGNN, automorphism groups, and EOL.
    \item We introduce CO-RTO, a new real-world application to tackle the challenge of the RTO problem in e-commerce. DPHGNN being robust to variations in topological constraints, beats HGNN baselines with a large margin.
\end{itemize}

\section{Preliminaries, Background and Related Work}
Let $HG = (V,\xi)$ denote an undirected hypergraph without self-loops, where $V$ is a set of hypernodes and a hyperedge $e \in \xi$ is composed of a set of nodes; therefore, $\{e \subset V\} \wedge \{ e \neq \emptyset\}$.  The incident edges of hypernode $i$ are denoted by \(E_{i}= \{e \in \xi \mid i \in e\}\). 
Table \ref{table1} summarizes the notations used throughout the paper.

\textbf{Graph Neural Networks.} Graph Neural Networks (GNNs) have been successful by following the message-passing neural network (MPNN) paradigm to update representations of nodes in the graph topology by aggregating the feature information from neighboring nodes. A simple message propagation in GNNs can be formulated as follows: spectral in
\begin{equation*}
\hat{x}_v^{(l)}=\operatorname{Update}^{(l)}\left(x_v^{(l)}, \operatorname{Aggr}^{(l)}\left(\left\{x_u^{(l-1)} \forall u \in N(v)\right\}\right)\right)
\end{equation*}
where  \(\operatorname{Aggr}^{(l)}\) aggregates node embeddings of the neighborhood $N(\cdot)$ of node $v$ at layer \(l-1\), and  \(\operatorname{Update}^{(l)}\) assigns aggregated embedding to node $v$ at layer $l$. 
$N(v)$ is the neighbors of node $v$, and $x_u^{(l-1)}$ is the representation of node $u$ at $(l-1)$ layer.
The foundation message-passing GNN models include GCN \cite{c:08}, GraphSAGE \cite{c:09}, GIN \cite{c:07}, JK-GCN \cite{c:31} and GAT \cite{c:10}.

\textbf{Spatial Hypergraph Networks.} To generalize MPNN over hypergraph, several attempts have been made in which spatial HGNNs tend to aggregate the hypernode embeddings from a hyperedge, aggregate the embedding from neighboring hyperedges and update the hypernode feature information \cite{c:11, c:12}. A simple spatial message propagation in HGNNs can be formulated as follows:


\begin{align} \label{eq:1}
\hat{x}_v^{(l)} &= \operatorname{Update}^{(l)}\left( x_v^{(l)}, \right. \nonumber \\
&\quad \left. \operatorname{Aggr}^{(l)}\left( \left\{ \left\{ \operatorname{Aggr}^{(l-1)}\left( x_u^{(l-1)} \right) \forall u \in e \right\} \forall e \in E_{i} \right\} \right) \right)
\end{align}

\noindent The benchmark models following spatial HGNN mechanism include HGNNP \cite{c:11}, HyperSAGE\cite{c:27}, HNHN \cite{c:12}, HAIN \cite{c:35} and the family of UniGNNs \cite{c:13}.

\textbf{Spectral Hypergraph Networks.} An HGNN layer uses spectral convolution to learn node/edge representation via approximating each cluster to a clique and requires quadratic computation. The HyperGCN layer \cite{c:14}  considers a linear number of edges resulting in a reduction in training times and uses a Laplacian operator for hypergraphs to approximate the hyperedges followed by a regular graph convolution operation on the resulting graph. MPNN-R \cite{c:15} considers hyperedges as new vertices \(V=\{V \cup \ \xi\}\) and presents the hypergraph with a $|V| \times |\xi|$ matrix. However, it fails to describe the intersections between hyperedges, resulting in a loss of complex semantic relations. These approaches can be implemented as smoothing features with trainable hypergraph Laplacian operators. Some frequently used Laplacian operators include
HGNN Laplacian matrix \(\Delta_{HGNN}=D_v^{-1 / 2} H D_e^{-1} H^T D_v^{-1 / 2}\), Symmetric Laplacian matrix \(\Delta_{sym}=I - D_v^{-1 / 2} H D_e^{-1} H^T D_v^{-1 / 2}\), Random walk Laplacian matrix \(\Delta_{rw}= I - D_v^{-1} H D_e^{-1} H^T\) \cite{c:26} (see Table \ref{table1} for notations).


\begin{algorithm}[!t]\small
\caption{DPHGNN: Dual Perspective Hypergraph Neural Network} \label{alg:algorithm1}
\begin{flushleft}
\textbf{Input} Hypergraph incidence matrix: $H \in R^{m \times n}$\ \\
               Hypernode feature matrix: $X_{H G}$ \\ 
               Target labels: $\{y\}_{\text{train}}$ \\
               Epochs: $e$ \\
               Num Layers: $k$ \\
\textbf{Output} Learned representations $\widehat{X}_{\text{HG}}$ for downstream tasks.
\end{flushleft}
\begin{algorithmic}
\For{\text{$\mathrm{i}=1$ to $e+1$ epochs}} 
    \State {\text{\textbf{Project:} $f(x): X_{H G} \rightarrow X_{H G}^{\prime}$ hidden dimension}}
    \State {\text{\textbf{Compute:} $\Delta_{\text{spectral}}=\left[\left(\Delta_{r w} \oplus \Delta_{\text {sym }}\right)|;| \Delta_{H G N N}\right]$}} 
    \State {\text{\textbf{Update hypernodes:} $X_{HG_{\text{spectral}}}=\sigma\left[\left(I+\lambda \Delta_{\text {spectral }}\right) X_{H G}^{\prime} \theta\right]$}} 
    \State {\text{\textbf{Compute TAA:} $\widehat{x}_\kappa=\Sigma_{\kappa \in N_\beta} \widehat{\alpha}_{\beta Y} \kappa, \widehat{x}_z=\Sigma_{Z \in N_x} \widehat{\alpha}_{X Y} Z$}} 
    \State {\textbf{Update hypernodes:}} 
    \State {$X_{\text{eqv}} = MLP^{1}(\hat{x}_{\kappa}|;|\hat{x}_{Z})\odot \sigma(\text{ReLU}(MLP^2(X_{HG_{\text{spectral}}})))$}
    \State {\textbf{Update static features}}
    \State {$X_{\text{static}} = X_{\text{eqv}} |;| MLP^{3}(X_{\text{HG}})$}
    \For{\text{$\mathrm{l}$ in $k-1$ layers}}
                \State Update hypernodes as:
                \State Feature Fusion:
                \State{$X_{\text{Fused}}=\left(H^T D_v^{-1 / 2} X_{\text {static}}+D_v^{-1} \phi_{\text {mask }}\left\{A X_G*\right\}\right)$}
                \State {$X_{\text {DPHGNN}}=\sigma\left[X_{\text {static}}+H D_e^{-1} X_{\text{Fused}}\Theta\right]$}
    \EndFor
    \State Update hypernodes as:
    \State {
                \resizebox{0.9\columnwidth}{!}{
                $\widehat{X}_{H G}=\sigma\left(D_v^{-1 / 2} X_{\text{DPHGNN}} \bar{D}_e^{-1 / 2} \cdot D_e^{-1} H^T D_v^{-1 / 2} X_{\text{DPHGNN}} \Theta\right)$}
                }
    \State Compute cross-entropy loss between $\{y\}_{\text {train }}$ and predictions $\left\{\widehat{X}_{H G}\right\}_{\text {train }}$
    \State Backpropagate loss and finetune parameter set using Adam optimizer
\EndFor
\end{algorithmic}
\end{algorithm}

\section{Dual Perspective Hypergraph Neural Network (DPHGNN)}
Our proposed DPHGNN architecture is presented in Figure \ref{fig:ourmethod} and summarized in Algorithm \ref{alg:algorithm1}. DPHGNN first induces various structural and spectral inductive biases to learn the underlying lower-order relations and spectral features. Following this, a dynamic feature fusion mechanism fuses explicit aggregated features from specific graph topology with the hypergraph message passing. We finetune different sub-layers in the static and dynamic induction blocks with cross-entropy loss by making the convolution layer an end-to-end differentiable pipeline. Specific building blocks are explained in the remaining section.

\subsection{Topology-Aware Attention (TAA)}
A hypergraph can be decomposed into various substructures representing underlying pairwise connections between hypernodes. We carefully decompose it into three particular graph topologies -- (i) {\bf Clique expansion} $G_{c}$ is used to learn the interconnection of node-pair entities.  (ii) {\bf Star expansion} $G_{*}$ is used to generate explicit supernodes utilized in the dynamic feature fusion module by masking the synthetic supernodes and nodes already present in the graph. This can be achieved by introducing a mask defined as,
\begin{align*}
\centering
\phi_{\text {mask }}=\left\{\begin{array}{l}1 \text { if } x_i \in V^* \cup S^* \\ 0 \text { if } x_i \in(V^* \cup S^*) \backslash V^*\end{array}\right.
\end{align*}
(iii) {\bf HyperGCN expansion}  $G_{hyp}$ is used to learn lower-order graph functions, approximate the hypergraph learning functions and update the feature representation using a single-layer MPNN  update. The choice of respective graph MPNN is made to maximize the induction of feature updates in respective graph topologies with the following rules:
\begin{align*}
\centering
X_c &= \sigma\left[\left(I+D_v^{-1} A_c\right) X \theta\right]; 
X_* = \sigma\left[\left(I+D_v^{-1} A_*\right) X \theta\right] \\
X_{hyp} &= \sigma\left(\widehat{D}_v^{-1 / 2} \widehat{A}_{h y p} \widehat{D}_v^{-1 / 2} X \theta\right)
\end{align*}
where $X_{c}$, $X_{*}$, and $X_{hyp}$ represent single-layer MPNN update of nodes on ${G_c}$, ${G_*}$, and $G_{hyp}$, respectively. We also compute smoothened features through graph Laplacian smoothing as \(L_{*}=D_{v}-A_{*}\), \(L_{c}=D_{v}-A_{c}\), \(L_{hyp}=D_{v}-A_{hyp}\) on \(G_{*}\), \(G_{c}\), and \(G_{hyp}\), respectively. This allows aggregating feature representations induced from different possible semantic relations among lower-order entities. Moreover, we adopt a cross-attention mechanism for graphs \cite{c:36} to generate topology-aware feature representations described below.

\begin{figure*}[!t]
\centering
\includegraphics[width=\textwidth]{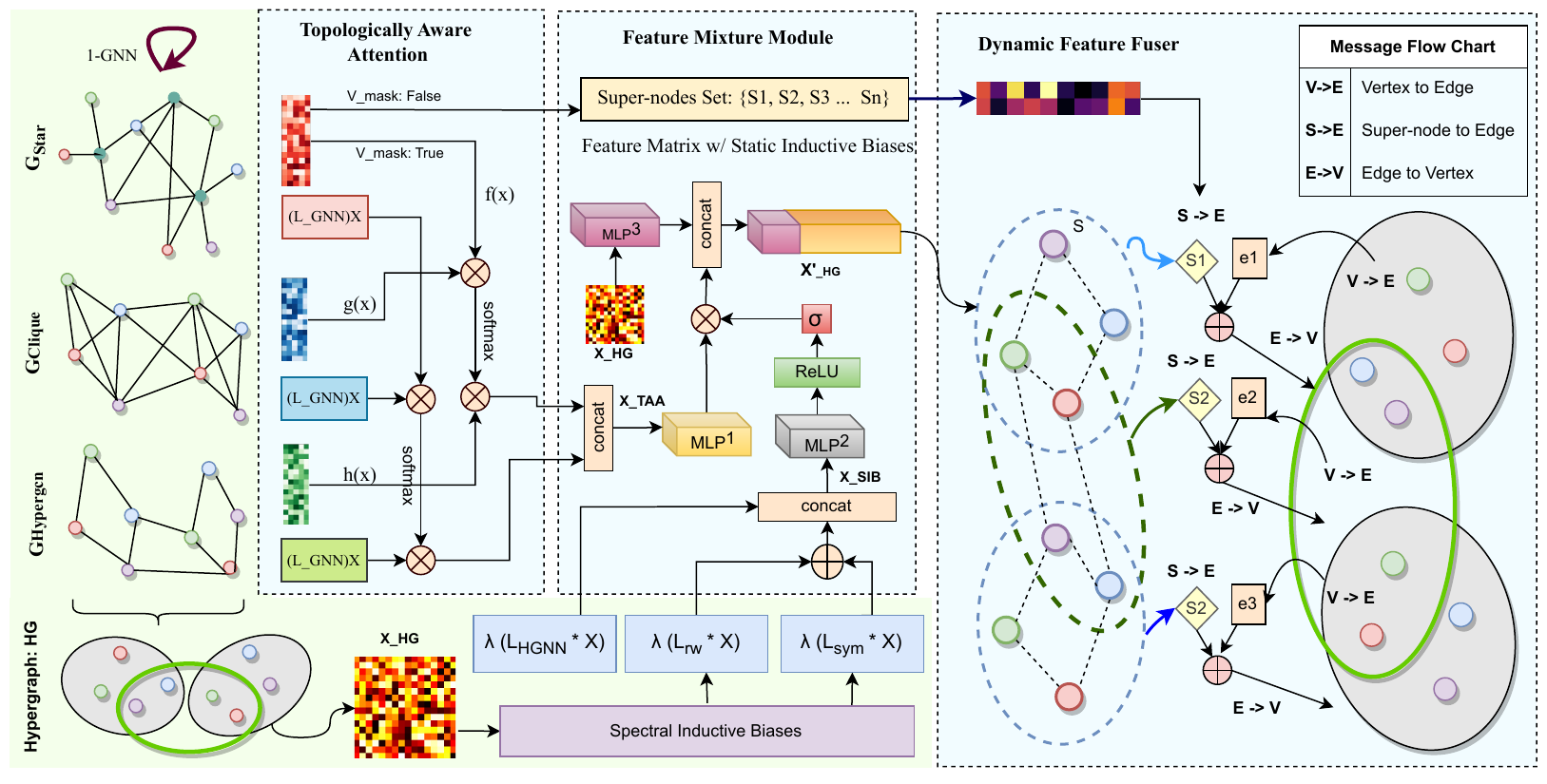} 
\caption{A schematic diagram of our proposed architecture, DPHGNN. Left: Hypergraph decomposition and topology-aware Attention (TAA) mechanism. Middle: Feature Mixture that generates static features by incorporating spectral inductive biases from hypergraph Laplacian smoothing and TAA.  Right: Dynamic feature fusion (DFF) that fuses explicitly learned graph embedding in supernodes with the hypergraph message-passing module.}
\label{fig:ourmethod}
\end{figure*}

\subsubsection{TAA on Spatial Features.} Let  \(\beta = \{\beta_{i}=x_{i} | \forall x_{i} \in \phi_{mask}(X_{*})\}\),  \(\gamma =\{\gamma_{j}=x_{j} | \forall x_{j} \in X_{c}\}\), and  \(\kappa= \{\kappa_{k}=x_{k} | \forall x_{k} \in X_{hyp}\}\)  be the sets of feature representation of decomposed graph topology $G_*$, $G_c$, and $G_{hyp}$. The topology-aware attention weights \(\widehat{\alpha}_{\beta\gamma}\) would be  along with trainable attention parameters \(\delta \in \mathbb{R}^{2d}\) with ``$;$'' denoting concatenation operation. The updated feature representation for \(G_{hyp}\) is then obtained by,
\begin{align*}
\centering
 \alpha_{\beta \gamma}&=\sigma\left(\delta^T[W\beta|;| W\gamma]\right);\ \widehat{\alpha}_{\beta \gamma}=\frac{\exp \left(\alpha_{\beta \gamma}\right)}{\sum_{p \in N_\beta} \exp \left(\alpha_{\beta \gamma}\right)};\\
\widehat{x}_\kappa &=\Sigma_{p \in N_\beta} \widehat{\alpha}_{\beta \gamma} W\kappa 
\end{align*}


\subsubsection{TAA on Spectral Features.} Let \(X = \{l_{i}=x_{i} | \forall x_{i} \in \phi_{mask}[L_{*}\otimes X_{*}]\}\),  \(Y = \{m_{j}=x_{j} | \forall x_{j} \in [L_{c}\otimes X_{c}]\}\), and \(Z = \{n_{k}=x_{k} | \forall x_{k} \in [L_{hyp}\otimes X_{hyp}]\}\) be the sets of Laplacian smoothen features of decomposed topologies $G_*$, $G_c$, and $G_{hyp}$, respectively. Similarly, the attention-weighted features for \(L_{hyp} \otimes X_{hyp}\) feature matrix would be obtained by,
\begin{align*}
\centering
 \alpha_{XY}&=\sigma\left(\delta^T[WX|;| WY]\right);\ 
 \widehat{\alpha}_{XY}=\frac{\exp \left(\alpha_{XY}\right)}{\sum_{p \in N_X} \exp \left(\alpha_{XY}\right)}; \\
 \widehat{x}_Z&=\Sigma_{p \in N_X} \widehat{\alpha}_{XY} W  Z 
\end{align*}
Here, $W$ represents a linear transformation matrix. After the topological attention-based re-weighting of the graph's structural and spectral features, we elapse them to the Feature Mixer Module.

\subsection{Feature Mixture Module} 
The sub-optimal performance of existing HGNNs, with the change in hypergraph topology, naturally motivates us to construct a generalized learning mechanism. To this end, we introduce a feature mixture module that first aggregates the spectral inductive biases (SIB) and then produces a mixture of features satisfying the following desiderata:
TAA concatenation operation (i) provides spatial and spectral features of lower-order graph topology for more rich semantic feature representations; (ii) prevents over-smoothing of relevant information at the time of message passing over sparse hypergraph topology. SIB concatenation operation provides symmetrical random walk-based HGNN Laplacian smoothen features, inherently giving unique identifiers to hypernodes. This helps us break automorphism groups.

{\bf Aggregating Spectral Inductive Biases}
DPHGNN leverages smoothened features from symmetrical random walk-based HGNN Laplacian, denoted by \(\Delta_{\text{spectral}} = \left[\left(\Delta_{\text{rw}} \oplus \Delta_{\text{sym}}\right) |;| \Delta_{\text{HGNN}}\right]\). The SIB block is inherently inspired by different approaches to formulating the hypergraph Laplacian, where the smoothened features \(X_{\text{HG}_{\text{spectral}}} = \sigma\left[X + \lambda \Delta_{\text{spectral}}X\right]\), and the concatenation operation \(|;|\) is crucial to prevent the nullification of spectral features. This is because the element-wise sum of Laplacian matrices \(\Delta_{\text{HGNN}} \oplus \Delta_{\text{sym}} \oplus \Delta_{\text{rw}}\) trivially results in \(3I -  D_v^{-1} H D_e^{-1} H^T\), i.e., it contains only the random-walk Laplacian information. These features provide unique identifiers to hypernodes, implicitly helping to break automorphism groups in hypergraph topology.

{\bf Equivariant Feature Mixing.} 
We employ an equivariant operator learning mechanism to perform static feature updates on the hypernode features. The aggregated hypernode representations from TAA and SIB modules are propagated through downscaling MLPs, \(MLP^{1}: \mathbb{R}^{n \times 2d} \rightarrow \mathbb{R}^{n \times d/2}\) and \(MLP^{2}: \mathbb{R}^{n \times 2d} \rightarrow \mathbb{R}^{n \times d/2}\) in parallel. The original features from hypernodes are propagated through \(MLP^{3}: \mathbb{R}^{n \times d} \rightarrow \mathbb{R}^{n \times d/2}\). We perform a permutation equivariant operation on aggregated features:
{\setlength{\abovedisplayskip}{2pt} 
\setlength{\belowdisplayskip}{2pt} 
\begin{equation} \label{eq:2}
    X_{\text{eqv}} = MLP^{1}(\hat{x}_{\kappa}|;|\hat{x}_{Z})\odot \sigma(\text{ReLU}(MLP^2(X_{HG_{\text{spectral}}})))
\end{equation}
}
where \(\odot\) represents Hadamard product; the static feature update finally upscales the dense information-rich features through a concatenation \(|;|\) operation as \(X_{\text{static}} = X_{\text{eqv}} |;| MLP^{3}(X_{\text{HG}})\).

\subsection{Dynamic Feature Fusion (DFF)}
Here, we describe the dynamic feature fusion mechanism to explicitly capture the lower-order semantics from pair-wise node interactions and diffuse it into the hypergraph message-passing step to learn generalizable feature representations. The star expansion produces graph topology $G_{*}$ by connecting existing nodes in the edges to a synthetic node for the hyperedge set \cite{c:34}. This allows us to extract aggregated information of a neighborhood in the graph in these synthetic nodes (called {\em supernodes}, hereafter). We utilize this feature representation from supernodes and fuse it into features aggregated from hypernodes towards hyperedges. We then update the representations in the neighborhood from the fused hyperedge towards subsequent hypernodes. Following this, we reiterate the feature updation without dynamic features on the hypergraph topology for the final prediction layer and perform class prediction. The dynamic message-passing mechanism described here can be formulated using the following update rules: \\  
\begin{align*}
m_p^{(l)}&=\phi_{\text{mask}}\left[\operatorname{Aggr}^{(l)}\left(\left\{m_u^{(l-1)} \forall u \in N(p)\right\}\right)\right] \\
m_e^{(l-1)}&=\operatorname{Aggr}^{(l-1)}\left(\left\{m_u^{(l-1)} \forall u \in e\right\} \cup\left\{m_p^{(l)} \forall p \in E\right\}\right) \\
m_e^{(l)}&=\operatorname{Aggr}^{(l)}\left(\left\{m_e^{(l-1)} \forall e \in E_i\right\}\right) \\
x_v^{(l)}&=\operatorname{Update}^{(l)}\left(\left\{x_v^{(l-1)}, m_e^{(l)}\right\}\right) \\
\end{align*}

The embedding update mechanism for hypernodes proposed in DPHGNN can be summarised below: 
\begin{equation}\small \label{eq:4}
\widehat{X}_{\text{DPHGNN}} = \sigma\left[X_{\text{static}} + H D_e^{-1}\left(H^T D_v^{-\frac{1}{2}} X_{\text{static}} + \\ D_e^{-1} \phi_{\text{mask}}\left\{A X_{G_*}\right\}\right)\Theta\right]
\end{equation}

We initialize the hypernode features with $X_{static}$. DPHGNN performs summation over supernodes and hypernodes at the time of aggregation. The dimensions of the DPHGNN matrix multiplications are as follows:
$H \in \mathbb{R}^{n \times m}$,  $H^{T} \in \mathbb{R}^{m \times n}$, $A \in \mathbb{R}^{n \times n}$. $X_{\text{static}} \in \mathbb{R}^{n \times d}$,  and $X_{G_{*}} \in \mathbb{R}^{n+m \times d}$ are input matrices. $D_{v} \in \mathbb{R}^{n \times n}$, and $D_{e} \in \mathbb{R}^{m \times m}$ are the diagonal matrices. The dimensions involved in the operations are $D_e^{-1} \phi_{\text {mask }}\{AX_{G_{*}}\}$, $\mathbb{R}^{m \times m} \times \mathbb{R}^{m \times d}\rightarrow\mathbb{R}^{m \times d}$, $H^T D_v^{-1 / 2} X_{\text {static }}$, $\mathbb{R}^{m \times m} \times (\mathbb{R}^{n \times n} \times \mathbb{R}^{n \times d})\rightarrow\mathbb{R}^{m \times d}$. At the DFF, $X \in \mathbb{R}^{m \times d}$, and finally, $\mathbb{R}^{n \times m} \times (\mathbb{R}^{m\times m} \times \mathbb{R}^{m \times d})\rightarrow\mathbb{R}^{n \times d}$.

\subsection{Time Complexity Analysis} Give an attributed hypergraph $\text{HG}(V, \xi, d_{v}, d_{e}, X)$, where $d_{v}$, $d_{e}$, and $X$ are average degree of $n$ hypernodes, average degree of $m$ hyperedges, and multiset of features $\{x_{v}\}_{v \in V}$, where $x_{v} \in \mathbb{R}^{d}$ is of dimension $d$. The computation time of decomposing hypergraph into multiple views of graph (clique, hypergcn, star) structure and applying graph convolution is $\mathcal{O}(n_{c}d_{v}d_{e}d+n_{hyp}d_{v}d+(n+s)_{*}d_{v}d_{e}d)$ $\leq$ $\mathcal{O}((n+s)_{*}d_{v}d_{e}d+n_{hyp}d_{v}d)$. For inductive priors, the compute time is upper bounded by $\mathcal{O}(n_{hyp}d_{v}d_{e}d+(n+m)_{hyp}d^{2})$. The time complexity for DPHGNN message update mechanism is $\mathcal{O} ((n+s)_{*}d_{v}d+n_{hyp}d+m_{hyp}d+n_{hyp}d^{2}+m_{hyp}d^{2})$ $\leq$ $\mathcal{O}((n+s)_{*}d_{v}d+n_{hyp}d^{2}+m_{hyp}d^{2})$, subscripts $*$, $c$, and $hyp$ denote the respective vertex set $G_{*}, G_{c}, \text{and}\ G_{hyp}$. Table \ref{tab:time_complexity} presents a comparative analysis between the proposed DPHGNN and baseline HGNNs. We also provide corresponding run-time analysis in Table \ref{tab:runtime} in the appendix.

\begin{table}[!t]
\centering
\resizebox{0.8\columnwidth}{!}{
\begin{tabular}{c|c}
\hline
{Model} & {Time Complexity} \\ \hline
HGNN & $\mathcal{O}(nd^{3})$ \\ 
HGNN+ & $\mathcal{O}(nd^{2}+md_{e}^{2})$ \\ 
HyperGCN & $\mathcal{O}(mn^{2}d^{2})$ \\ 
HNHN & $\mathcal{O}(nd^{2} + md^{2})$ \\ 
UniGCN & 	$\mathcal{O}(md+nd^{2})$ \\ \hdashline
DPHGNN & $\mathcal{O}((n+s_{*})d_{v}d+n_{hyp}d^{2}+m_{hyp}d^{2})$ \\ \hline
\end{tabular}}
\caption{Time Complexity of different models.}
\label{tab:time_complexity}
\end{table}

\section{Theoretical Characterization}
\subsection{Equivariant Operator Learning}
The DPHGNN architecture generalizes the performance across the change in hypergraph topology; several operators are used to bridge the tradeoffs of spatial and spectral HGNNs. In proposition \ref{prop:4.1} we prove that the feature mixer module encodes features from TAA and SIB blocks in a permutation equivariant manner.
We have provided an in-depth analysis of the correlation between specified model blocks and hypergraph dataset in the Ablation Study section in Experiments.

\begin{proposition}
    The encoding function $f: (X_{\text{HG}}, X_{\text{TAA}}, X_{\text{SIB}})\rightarrow X_{\text{static}}$ learned by equation \ref{eq:2} is permutation equivariant, i.e, if $\pi$ is a bijective function; $f(\pi \cdot X)= \pi \cdot f(X)$.
    \label{prop:4.1}
\end{proposition}
\begin{proof} 
    Let $A = MLP^{1}(\hat{x}_{\kappa}|;|\hat{x}_{Z})$, and $B = (MLP^2(X_{HG_{\text{spectral}}}))$ are the matrices obtained by equation \ref{eq:2} with elements $a_{ij}$, $b_{ij}$ respectively. Let $a'_{ij} = a_{\pi(i)\pi(j)}$ and $b'_{ij} = b_{\pi(i)\pi(j)}$ are elements of matrics A, B after permutation $\pi$. The Hadamard product matrix (C) of elements results in $c'_{ij}= a'_{ij} \odot b'_{ij}$. The original Hadamard Product $C=A \odot B$ has elements $c_{ij}=a_{ij} \odot b_{ij}$, applying permutation $\pi$ to C will result in the same operation as with $\pi$ to original matrices, i.e. $c_{\pi(i)\pi(j)} = a_{\pi(i)\pi(j)} \odot b_{\pi(i)\pi(j)}$. Hence, function $f$ is permutation equivariant. 
\end{proof}

\subsection{Expressive Power of DPHGNN} \label{ss: 4.1}
The expressive power of HGNNs is determined by their ability to learn a function on hypergraph that can distinguish two non-isomorphic hypergraphs \cite{c:32} and their local substructures. \cite{c:13} formulated a variant of the Generalized Weisfeiler Leman (1-GWL) test for measuring the expressive power of UniGNNs, following \cite{c:16}. However, the increase in symmetry from automorphisms in graph/hypergraph structure GWL inherits several failure cases in distinguishing complex substructures. In Proposition \ref{prop:4.2}, we generalize the 1-GWL test for spatial HGNNs which follow message passing (c.f. Eq. \ref{eq:1}) through the lens of the hypergraph color refinement algorithm.
We then analyze the expressive power of DPHGNN in theorem \ref{th:4.3} and prove that providing explicit representation information of underlying graph structure helps break automorphisms \cite{c:38} via learning equivariant functions.

\begin{proposition}
     Given a function $f_{\theta} \in F^{\text{HGNN}}$ (learned by Eq. \ref{eq:1}), and two non-isomorphic hypergraphs, $H_{1}$ and $H_{2}$, $f_{\theta}$ can distinguish $f_{\theta}(H_{1}) \neq f_{\theta}(H_{2})$ if and only if for some $t>0$, the updated coloring 
$H C_{(t)}\left(V_1, H_1\right) \neq H C_{(t)}\left(V_2, H_2\right)$.
\label{prop:4.2}
\end{proposition}

\paragraph{Proof Sketch.} 
Here, we provide the skeleton of the proof (see \ref{proof:p2} Supplementary for the detailed proof).  We first generalize the color refinement for hypergraphs by providing tensor representations for colors. We then establish an injective mapping between the  HASH function of iterative hypergraph color refinement (c.f. Eq. \ref{eq:5}) and general spatial HGNNs update (c.f. Eq. \ref{eq:1}) 
\begin{multline}\label{eq:5}
HC_{t+1}^H(v)=\text{HASH}\{\{\{{H C_t^H(u) \mid u \in f_H(e)}\} \\ \mid e \in \xi \ \text{with}\ v \in f_H(e)\}\}
\end{multline}

\begin{theorem}
        Given a function $g_{\theta} \in F^{\text{DPHGNN}}$ (learned by Eq. \ref{eq:4}), and two non-isomorphic hypergraphs $H_{1}$ and $H_{2}$ that can be distinguished by 3-GWL test, there is at least a function $g_{\theta}$ that can also decide $g_{\theta}(H_{1}) \neq g_{\theta}(H_{2})$ if $g_{\theta}$ is equipped with the following structure: (i) $g_{\theta}$ is an order invariant function learned by the composition of $k$-equivariant layer $(k\geq3)$ HGNN network. (ii) Aggregation and Readout functions of DPHGNN are injective.
        \label{th:4.3}
\end{theorem}

\paragraph{Proof Sketch.} Here, we provide the skeleton of the proof (see \ref{proof:t1} Supplementary for the detailed proof). We first build upon the work of \citet{c:17, c:33} to generalize the notion of the $F$-WL (Folklore-WL) test \cite{c:39} for hypergraphs. We then argue that using explicit feature representations from underlying graph structure as the composition of equivariant layers is at least as powerful as the $F$-GWL test. We then prove the second part of the theorem.

\begin{table}[!t]
\centering
\resizebox{0.9\columnwidth}{!}{
\begin{tabular}{lcccccc}
\hline
Dataset         & $|V|$ & $|\xi|$  & $|\mu|$ & $|M|$      & $d$   & $C$ \\ \hline
CA-DBLP         & 2708  & 1072   & 0.79 & 4.2   & 1433 & 7  \\
CA-Cora         & 43413 & 22535  & 1.03 & 4.7   & 1425 & 6  \\
CC-Citeseer     & 2708  & 1579   & 1.16 & 3.0   & 1433 & 6  \\
CC-Cora         & 3312  & 1079   & 1.03 & 3.2   & 3703 & 6  \\
YelpRestaurant  & 50758 & 679302 & 26.7 & 6.6   & 1862 & 11 \\
HouseCommittees & 1290  & 341    & 0.52 & 34.7 & 1290 & 3  \\
Cooking200      & 7403  & 2755   & 0.74 & 19.9 & 7403 & 20 \\
News20          & 16342 & 100    & 0.01 & 654.5        & 100  & 4  \\ 
CO-RTO          & 66790 & 27528  & 0.82 & 1.6   & 10   & 2  \\\hline
\end{tabular}}
\caption{Dataset statistics: $|V|$ and $|\xi|$ represent the number of hypernodes and hyperedges, respectively. $|\mu|= \frac{2|\xi|}{|V|}$ is. the average hyperedge density, and $|M|= \frac{1}{|\xi|} \sum_{e \in \xi}\nolimits |e|$ is the average hypernode density of hyperedges; $d$ and $C$ are the dimension of features and the number of classes, respectively.}
\label{table2}
\vspace{-5mm}
\end{table}

\begin{table*}
\centering
\resizebox{2.0\columnwidth}{!}{
\begin{tabular}{ccccccccccccc}
\hline
\multicolumn{1}{c|}{\bf Methods}                        &                                  & {\bf CA-DBLP}                                                   & \multicolumn{1}{c|}{}                                 &                                  & {\bf CA-Cora}                                                   & \multicolumn{1}{c|}{}                                 &                                  & {\bf CC-Citeseer}                      & \multicolumn{1}{c|}{}                                 &                                  & {\bf CC-Cora}                          &                                  \\ \cline{2-13} 
\multicolumn{1}{c|}{}                              & Mean Acc                         & Macro F1                                                  & \multicolumn{1}{c|}{Micro F1}                         & Mean Acc                         & Macro F1                                                  & \multicolumn{1}{c|}{Micro F1}                         & Mean Acc                         & Macro F1                         & \multicolumn{1}{c|}{Micro F1}                         & Mean Acc                         & Macro F1                         & Micro F1                         \\ \hline
\multicolumn{1}{c|}{HGNN}                          & 68.24±9.8                        & 67.61±8.5                                                 & \multicolumn{1}{c|}{68.24±9.9}                        & 63.92±4.2                        & 53.88±5.4                                                 & \multicolumn{1}{c|}{64.01±5.2}                        & {\color[HTML]{6434FC} 68.56±0.3} & {\color[HTML]{6434FC} 60.52±1.1} & \multicolumn{1}{c|}{{\color[HTML]{6434FC} 68.56±0.2}} & {\color[HTML]{FE0000} 69.38±3.8} & {\color[HTML]{FE0000} 69.03±3.9} & {\color[HTML]{FE0000} 68.92±5.4} \\
\multicolumn{1}{c|}{HGNN+}                         & {\color[HTML]{CD9934} 84.66±5.1} & 81.12±1.8                                                 & \multicolumn{1}{c|}{{\color[HTML]{CD9934} 84.66±5.1}} & 68.76±4.8                        & 61.90±5.2                                                 & \multicolumn{1}{c|}{69.76±3.6}                        & 42.02±5.3                        & 37.21±3.0                        & \multicolumn{1}{c|}{42.02±6.2}                        & 46.26±5.2                        & 49.3±3.2                         & 46.26±5.2                        \\
\multicolumn{1}{c|}{HyperGCN}                      & 84.03±5.2                        & 83.29±4.2                                                 & \multicolumn{1}{c|}{84.01±6.2}                        & 63.91±7.1                        & 51.22±9.2                                                 & \multicolumn{1}{c|}{68.99±5.5}                        & 63.2±1.8                         & {\color[HTML]{CD9934} 60.24±2.2} & \multicolumn{1}{c|}{63.2±1.1}                         & {\color[HTML]{6434FC} 67.07±2.1} & {\color[HTML]{6434FC} 66.52±4.9} & {\color[HTML]{6434FC} 67.87±2.1} \\
\multicolumn{1}{c|}{HNHN}                          & 65.32±1.4                        & 61.09±2.3                                                 & \multicolumn{1}{c|}{65.32±1.5}                        & 65.28±4.2                        & 54.11±7.3                                                 & \multicolumn{1}{c|}{68.91±4.8}                        & 33.45±4.5                        & 31.65±3.9                        & \multicolumn{1}{c|}{33.45±4.5}                        & 48.03±4.3                        & 42.55±0.5                        & 48.03±4.3                        \\
\multicolumn{1}{c|}{UniGCN}                        & {\color[HTML]{6434FC} 88.07±0.1} & {\color[HTML]{6434FC} 86.03±0.4}                          & \multicolumn{1}{c|}{{\color[HTML]{6434FC} 88.07±0.2}} & {\color[HTML]{CD9934} 71.68±4.4} & {\color[HTML]{CD9934} 57.31±6.2}                          & \multicolumn{1}{c|}{{\color[HTML]{CD9934} 79.74±1.3}} & 52.47±8.2                        & 47.54±4.8                        & \multicolumn{1}{c|}{52.47±8.2}                        & 61.1±7.4                         & 55.24±8.2                        & 61.1±7.4                         \\
\multicolumn{1}{c|}{UniSAGE}                       & 87.97±0.2                        & 85.37±0.1                                                 & \multicolumn{1}{c|}{87.96±0.4}                        & 72.54±3.4                        & 57.26±6.4                                                 & \multicolumn{1}{c|}{78.69±3.4}                        & 57.59±7.5                        & 55.2±3.2                         & \multicolumn{1}{c|}{57.59±7.2}                        & 64.87±6.0                        & 63.26±1.5                        & 64.87±6.0                        \\
\multicolumn{1}{c|}{UniGAT}                        & 86.69±0.1                        & 84.24±1.8                                                 & \multicolumn{1}{c|}{87.97±1.1}                        & 72.48±2.8                        & 58.69±4.5                                                 & \multicolumn{1}{c|}{75.2±1.8}                         & 52.68±8.4                        & 46.95±2.9                        & \multicolumn{1}{c|}{51.96±8.4}                        & 55.83±6.5                        & 51.58±4.2                        & 55.83±6.5                        \\ \hline
\multicolumn{1}{c|}{{\color[HTML]{333333} DPHGNN}} & {\color[HTML]{FE0000} 89.06±1.2} & {\color[HTML]{FE0000} 86.48±2.5}                          & \multicolumn{1}{c|}{{\color[HTML]{FE0000} 87.16±1.8}} & {\color[HTML]{FE0000} 75.12±1.2} & {\color[HTML]{FE0000} 68.44±1.9}                          & \multicolumn{1}{c|}{{\color[HTML]{FE0000} 78.94±3.3}} & {\color[HTML]{FE0000} 65.77±4.2} & 56.52±3.5                        & \multicolumn{1}{c|}{{\color[HTML]{FE0000} 63.77±5.4}} & {\color[HTML]{CD9934} 67.51±6.2} & {\color[HTML]{CD9934} 64.55±8.2} & {\color[HTML]{CD9934} 68.51±3.4} \\ \hline
\multicolumn{1}{l}{}                               & \multicolumn{1}{l}{}             & \multicolumn{1}{l}{}                                      & \multicolumn{1}{l}{}                                  & \multicolumn{1}{l}{}             & \multicolumn{1}{l}{}                                      & \multicolumn{1}{l}{}                                  & \multicolumn{1}{l}{}             & \multicolumn{1}{l}{}             & \multicolumn{1}{l}{}                                  & \multicolumn{1}{l}{}             & \multicolumn{1}{l}{}             & \multicolumn{1}{l}{}             \\ \hline
\multicolumn{1}{c|}{\bf Methods}                        &                                  & \begin{tabular}[c]{@{}c@{}}\bf Yelp Restaurant\end{tabular} & \multicolumn{1}{c|}{}                                 &                                  & \begin{tabular}[c]{@{}c@{}}\bf House Committees
\end{tabular} & \multicolumn{1}{c|}{}                                 &                                  & \bf Cooking200                       & \multicolumn{1}{c|}{}                                 &                                  & \bf News20                           &                                  \\ \cline{2-13} 
\multicolumn{1}{c|}{}                              & Mean Acc                         & Macro F1                                                  & \multicolumn{1}{c|}{Micro F1}                         & Mean Acc                         & Macro F1                                                  & \multicolumn{1}{c|}{Micro F1}                         & Mean Acc                         & Macro F1                         & \multicolumn{1}{c|}{Micro F1}                         & Mean Acc                         & Macro F1                         & Micro F1                         \\ \hline
\multicolumn{1}{c|}{HGNN}                          & 30.69±3.2                        & {\color[HTML]{6434FC} 14.37±6.7}                          & \multicolumn{1}{c|}{28.24±2.8}                        & {\color[HTML]{333333} 52.33±5.1} & {\color[HTML]{333333} 50.06±5.5}                          & \multicolumn{1}{c|}{{\color[HTML]{333333} 50.03±3.2}} & {\color[HTML]{6434FC} 53.45±2.6} & {\color[HTML]{6434FC} 40.91±5.8} & \multicolumn{1}{c|}{{\color[HTML]{6434FC} 50.83±2.9}} & 80.46±3.2                        & 78.74±0.4                        & 80.46±2.4                        \\
\multicolumn{1}{c|}{HGNN+}                         & {\color[HTML]{6434FC} 32.92±1.8} & 12.84±4.2                                                 & \multicolumn{1}{c|}{{\color[HTML]{6434FC} 30.53±1.5}} & 53.1±5.9                         & 52.58±4.2                                                 & \multicolumn{1}{c|}{51.27±3.5}                        & 49.64±1.2                        & {\color[HTML]{CD9934} 40.07±8.6} & \multicolumn{1}{c|}{48.27±1.5}                        & {\color[HTML]{CD9934} 81.25±2.8} & {\color[HTML]{CD9934} 78.89±1.1} & {\color[HTML]{CD9934} 81.25±1.2} \\
\multicolumn{1}{c|}{HyperGCN}                      & 30.77±1.1                        & 13.3±2                                                    & \multicolumn{1}{c|}{28.41±3.1}                        & NA                               & NA                                                        & \multicolumn{1}{c|}{NA}                               & 9.32±2.8                         & 4.28±1.6                         & \multicolumn{1}{c|}{6.64±2.6}                         & 81.16±0.8                        & 78.95±0.5                        & 80.16±1.8                        \\
\multicolumn{1}{c|}{HNHN}                          & 26.61±4.4                        & 7.64±1.2                                                  & \multicolumn{1}{c|}{24.75±2.5}                        & 49.61±4.2                        & 47.96±1.4                                                 & \multicolumn{1}{c|}{46.68±4.2}                        & 25.29±4.7                        & 20.1±4.7                         & \multicolumn{1}{c|}{23.33±3.3}                        & 76.87±5.4                        & 78.75±0.3                        & 81.13±0.6                        \\
\multicolumn{1}{c|}{UniGCN}                        & 28.56±2.5                        & 10.38±0.9                                                 & \multicolumn{1}{c|}{26.6±4.2}                         & {\color[HTML]{CD9934} 56.04±5.5} & {\color[HTML]{CD9934} 54.87±2.5}                          & \multicolumn{1}{c|}{{\color[HTML]{CD9934} 53.32±4.8}} & 49.49±3.3                        & 38.53±1.6                        & \multicolumn{1}{c|}{47.43±1.4}                        & 80.97±0.5                        & 78.82±0.1                        & 80.97±0.5                        \\
\multicolumn{1}{c|}{UniSAGE}                       & 27.3±0.4                         & 10.41±1.5                                                 & \multicolumn{1}{c|}{25.93±1.6}                        & {\color[HTML]{FE0000} 61.72±8.4} & 50.87±6.7                                                 & \multicolumn{1}{c|}{{\color[HTML]{FE0000} 61.21±0.5}} & {\color[HTML]{CD9934} 52.32±3.9} & 42.24±2.5                        & \multicolumn{1}{c|}{49.64±2.2}                        & {\color[HTML]{6434FC} 81.49±1.5} & {\color[HTML]{6434FC} 79.47±1.8} & {\color[HTML]{6434FC} 81.49±0.3} \\
\multicolumn{1}{c|}{UniGAT}                        & OOM                              & OOM                                                       & \multicolumn{1}{c|}{OOM}                              & {\color[HTML]{6434FC} 57.36±3.1} & {\color[HTML]{6434FC} 56.31±0.2}                          & \multicolumn{1}{c|}{{\color[HTML]{6434FC} 54.73±2.9}} & 34.99±1.8                        & 26.17±3.6                        & \multicolumn{1}{c|}{33.73±0.9}                        & 78.76±4.4                        & 78.15±0.3                        & 80.76±3.3                        \\ \hline
\multicolumn{1}{c|}{DPHGNN}                        & {\color[HTML]{FE0000} 35.82±1.4} & {\color[HTML]{FE0000} 20.55±0.3}                          & \multicolumn{1}{c|}{{\color[HTML]{FE0000} 36.33±1.2}} & {\color[HTML]{CD9934} 56.87±4.3} & 51.72±0.3                                                 & \multicolumn{1}{c|}{{\color[HTML]{CD9934} 54.33±3.3}} & {\color[HTML]{FE0000} 55.74±2.1} & {\color[HTML]{FE0000} 46.51±1.5} & \multicolumn{1}{c|}{{\color[HTML]{FE0000} 51.73±3.4}} & {\color[HTML]{FE0000} 81.57±1.3} & {\color[HTML]{FE0000} 79.2±0.5}  & {\color[HTML]{FE0000} 81.43±0.3} \\ \hline
\end{tabular}}
\caption{Performance comparison on the benchmark  hypergraph datasets over mean accuracy(\%), macro F1-score, and micro F1-score (± standard deviation). For every dataset, the best performance is highlighted in \textcolor{red}{red}, the best baseline is highlighted in \textcolor{blue}{blue}, results within 1-std. dev. are highlighted in \textcolor{brown}{brown}. NA indicates structural restriction for method; OOM is out of memory.}
\label{table3}
\vspace{-3mm}
\end{table*}

\section{Experiments}
This section presents benchmark datasets, baseline methods, comparative analyses and ablation studies of our model. 


 {\bf Benchmark Datasets.} We experiment with publicly available benchmark datasets -- Cora, DBLP, and Citeseer, extensively used in previous studies \cite{c:13, c:14}. Two standard ways to construct hypergraphs from these datasets result in coauthorship (CA) networks from Cora and DBLP and cocitation networks (CC) from Cora and Citeseer. Moreover, we also benchmark on the pre-constructed standard hypergraph datasets -- YelpRestaurant, HouseCommittees, Cooking200, and News20 \cite{c:41}. We utilize the hypergraph convolution layers from the DeepHyperGraph (DHG) package for our baseline experiments. Table \ref{table2} provides detailed statistics of these datasets.

{\bf Baseline Methods.}
We compare DPHGNN with seven standard baselines designed for semi-supervised node classification on hypergraphs. The baselines include a mix of spatial and spectral approaches for hypergraph representation learning -- 
(i) HGNN: It parametrizes the smoothing of features from HGNN Laplacian matrix \cite{c:06}; (ii) HGNN+: It uses adaptive hyperedge group fusion strategy to introduce spatial message passing for different modalities \cite{c:11}; (iii) HyperGCN: It introduces Laplacian operator to approximate underlying graph structure and perform message passing over lower-order graph \cite{c:14}; (iv) HNHN: Spatial HGNN with non-linear activation functions are applied over both hypernodes and hyperedges \cite{c:12}; UniGNN \cite{c:14} generalises spatial learning GNNs for hypergraphs and formulates various architectures such as (v) UniGCN, (vi) UniSAGE, and (vii) UniGAT.

{\bf Experimental Setup.}
For DPHGNN, we deploy a dual-layer training mechanism (see Figure \ref{fig:training_pipeline} in Supplementary), in which the static layer is trained to learn the underlying graph structural and hypergraph spectral inductive biases, and the dynamic layer is used to optimize learning on hypergraphs. 
We use a suitable GNN layer respective to the decomposed structure for graph structural learning and update the feature representation for the 2-hop neighborhood. Moreover, we use two-layer message passing for hypergraph structural learning, one for updating the neighbor hypernode features and the other for a class-wise prediction head. (Table \ref{table7} in Supplementary includes specific hyperparameter details for our experiments).

\begin{figure*}[!t]
    \centering
    \begin{minipage}{0.8\textwidth}
        \includegraphics[width=0.33\textwidth]{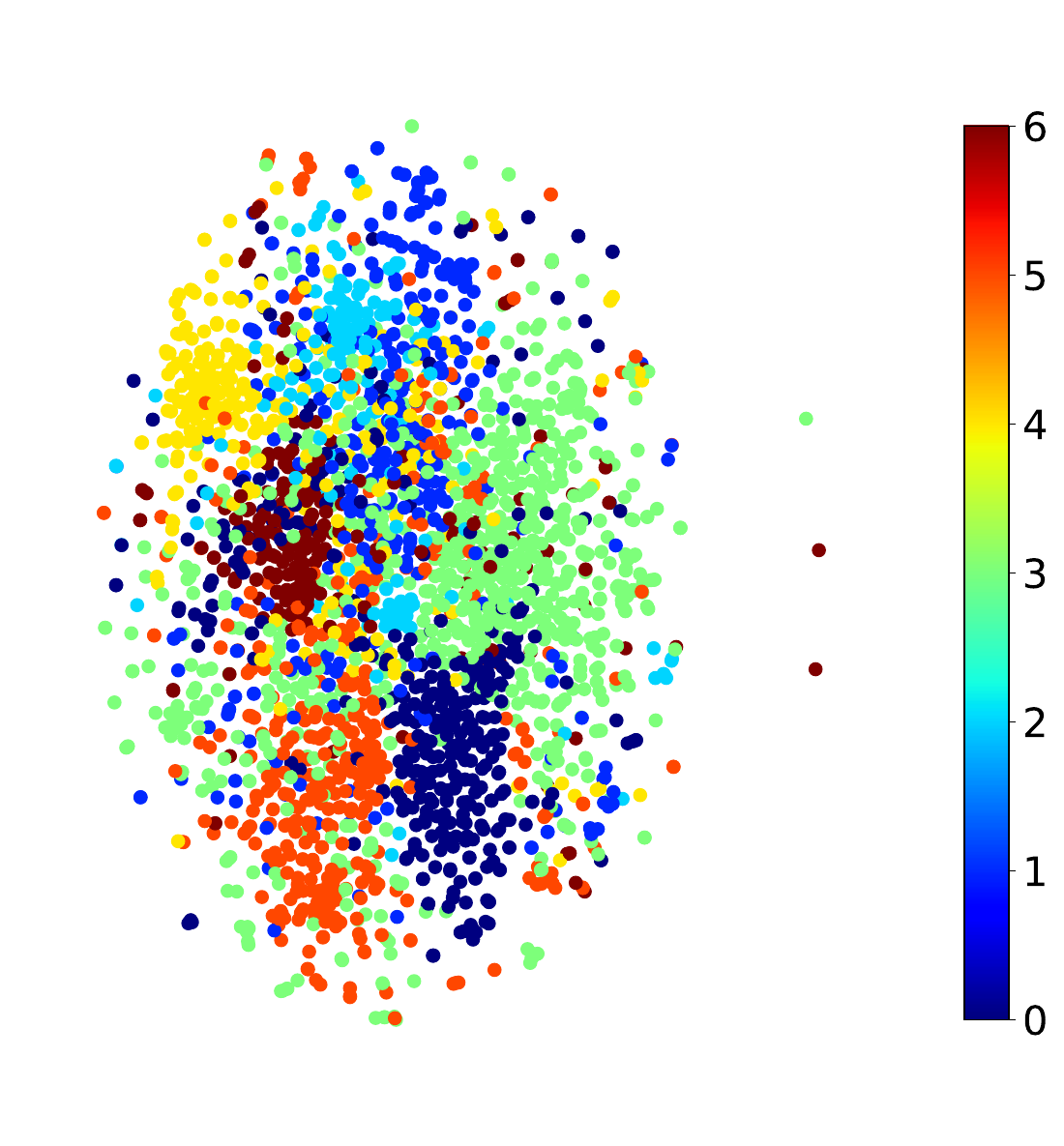}\hfill
        \includegraphics[width=0.33\textwidth]{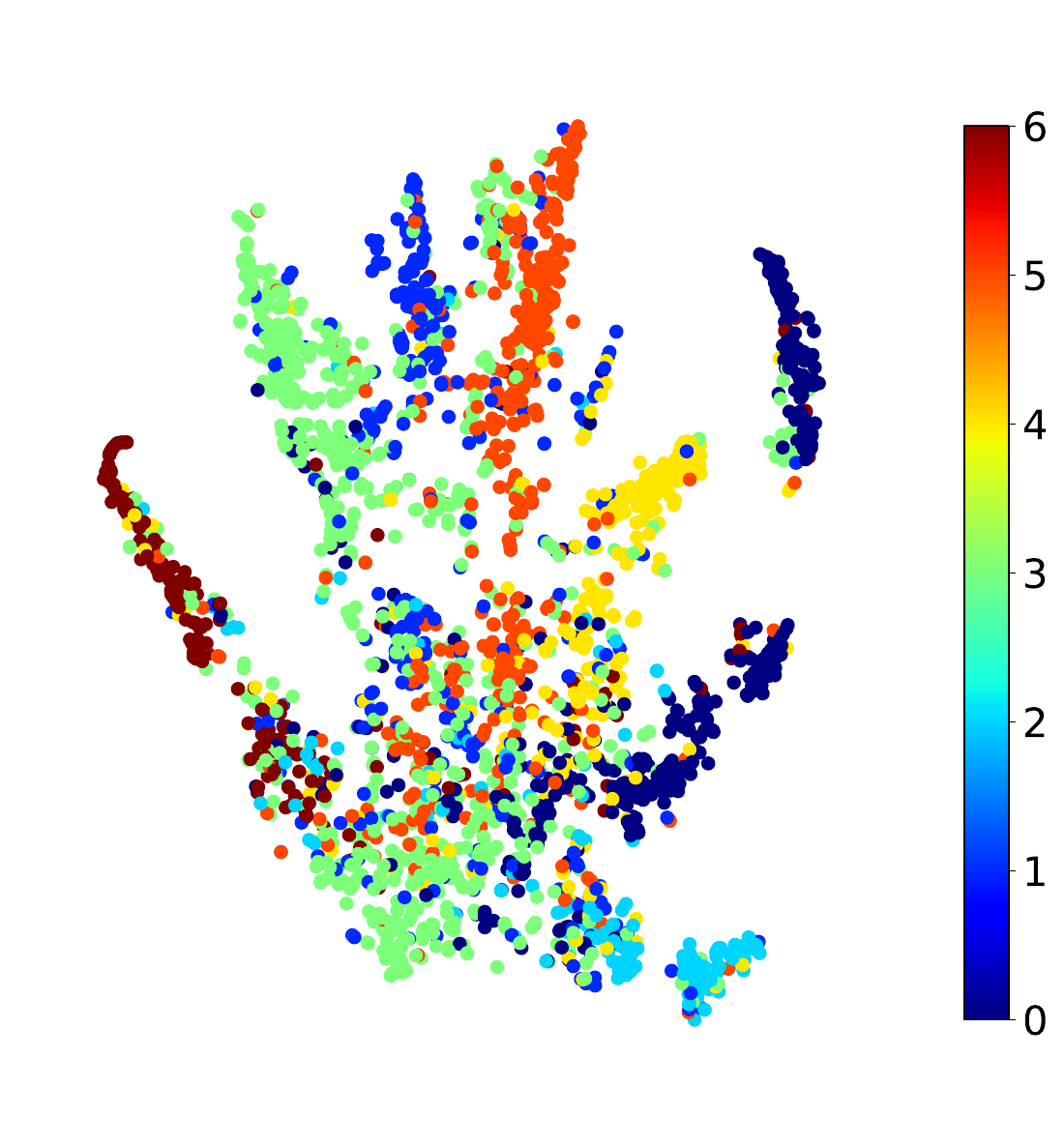}\hfill
        \includegraphics[width=0.33\textwidth]{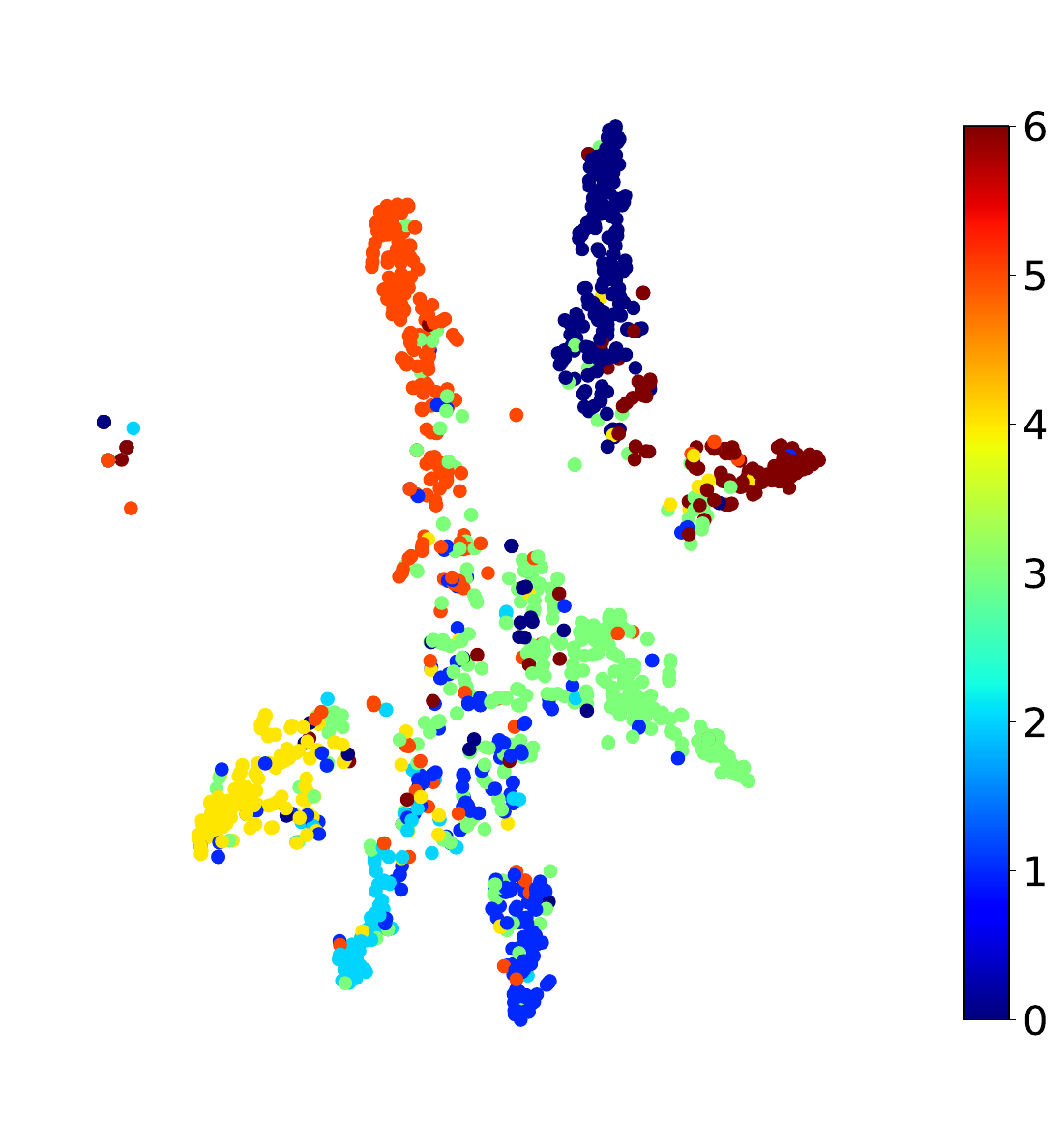}
    \end{minipage}%
    \vspace{-5mm}
        \caption{Visualization of feature embedding update on the CC-Citeseer dataset (6 classes) -- initial embedding (left), HGNN embedding update (middle), DPHGNN embedding update (right) (see Section \ref{sec:viz} for other datasets).}
    \label{fig:CCCiteseer}
    \vspace{-3mm}
\end{figure*}

\subsection{Performance Comparison}
We evaluate the model performance based on mean accuracy, macro F1-score, and micro F1-score. We report the average performance of the model over ten runs, along with the standard deviation.

\subsubsection{Comparative Analysis.} Table \ref{table3} presents the comparative analysis. We observe that spatial HGNN methods outperform the spectral ones for co-authorship datasets (by a margin of $\sim$$5\%$). However, it is the opposite in the case of co-citation datasets in which spectral methods stand out (by a significant margin of $7\%$-$9\%$). 

Our proposed DPHGNN method, with TAA, SIB, and DFF+HGNN modules, is robust to the change in hypergraph topology. We quantify the performance improvement in DPHGNN by comparing it with the best-performing spatial HGNN on co-authorship datasets. DPHGNN outperforms UniGCN, UniSAGE, and UniGAT by $1.1\%$, $1.2\%$, and $2.3\%$, respectively, on CA-DBLP, and $3.4\%$, $2.5\%$, and  $2.6\%$, respectively, on CA-Cora. For the co-citation datasets, with the induction of SIB and TAA (spectral), DPHGNN outperforms spatial UniGNNs (UniGCN, UniSAGE, UniGAT) by $6.3\%$, $4.2\%$, and $9.6\%$, respectively. On CC-Citeseer, DPHGNN outperforms the best baseline by $2.2\%$ and gives the best performance on F1-Score within one standard deviation (std) range. On CC-Cora, HGNN, being a fully spectral baseline, surpasses DPHGNN by a small margin of $1.8\%$; however, DPHGNN outperforms the rest of the spatial baselines, namely UniGCN, UniSAGE and UniGAT, with a significant margin of $6.4\%$, $2.6\%$, $11.6\%$, respectively. For HouseCommittees, UniGNNs surpass DPHGNN by a small margin of $2.8\%$ (within mean std range). On YelpRestaurant, Cooking200, and News20 with dense hypernode distributions, DPHGNN outperforms best spatial baselines (HGNN+, UniSAGE, UniGAT) with $2.9\%$, $3.3\%$, $0.8\%$ and best spectral baselines (HGNN, HGNN, HyperGCN) with $5.5\%$, $2.3\%$, $1.4\%$ respectively. The ablation study below empirically validates the effect of each module in improving the performance of DPHGNN.

{\bf Embedding Update Visualization.}
The performance of a classification model is inherently dependent on interpolating the underlying distribution and clustering the data points for respective classes. Figure \ref{fig:CCCiteseer} and Section \ref{sec:viz} provide t-SNE visualization \cite{c:42} of embedding update from DPHGNN and that from the best baseline (HGNN). We observe the former one segregating classes more effectively than the latter one.  
 
\subsubsection{Ablation Study. \label{sec:ablation}}
We study the impact of each module of DPHGNN and report the change in the model's performance with the structure of hypergraph datasets in Table \ref{table:4}. 
\begin{itemize}[leftmargin=*]
    \item On co-authorship datasets (CA-DBLP, CA-Cora), most impact is due to spatial information flow. Without TAA, the performance of DPHGNN decreases ($1.3\%\downarrow$, $1.5\%\downarrow$); however, DFF propagates the lower-order structural information. Without DFF ($2.5\%\downarrow$, $2.7\%\downarrow$), TAA can only provide cross-attention information of node entities as inductive bias.
    \item On co-citation datasets, CC-Citeseer, and CC-Cora, the TAA (spectral) and SIB modules of DPHGNN have the majority of contributions, without which the performance of DPHGNN deteriorates by $6.3\%\downarrow$, and $5.9\%\downarrow$, respectively. This supports the empirical analysis of spectral HGNNs performing better than spatial baselines over co-citation datasets. Without TAA, the performance of DPHGNN degrades ($2.9\%\downarrow$, $3.3\%\downarrow$) as TAA propagates spectral inductive biases.
    \item On HouseCommittees and Cooking200, features are initialised with one-hot encoding of node indices. Therefore, TAA and SIB have the most impact on empowering features with spatial and spectral inductive biases; the performance deteriorates without TAA ($4.1\%\downarrow$, $5.7\%\downarrow$) and without SIB ($3.2\%\downarrow$, $~6.2\%\downarrow$). Moreover, due to sparse connectivity ($|\mu| \sim 0.5,0.7$) with dense hypernode density ($|M| \sim 34, 19$) TAA (spectral) complements in the absence of SIB, and TAA (spatial) for DFF.
    \item On News20, most of the HGNNs suffer from over-smoothing of features as the highly dense hypernodes ($|M|\sim 654$) diffuse features via highly sparse hyperedges ($|\mu| \sim 0.01$), resulting in similar performance irrespective of HGNN architecture. Without SIB ($4.5\%\downarrow$), DPHGNN struggles to break automorphisms in this symmetric structure.
    \item On YelpRestaurant, with densely connected hypergraph, the performance drops without TAA and DFF ($5.2\%\downarrow$, $3.3\%\downarrow$) as they provide a major contribution in the spatial message flow. Without SIB ($1.3\%\downarrow$), TAA (spectral) provides desired inductive biases.
\end{itemize}

\begin{figure}[!t]
\centering
\includegraphics[width=0.5\textwidth]{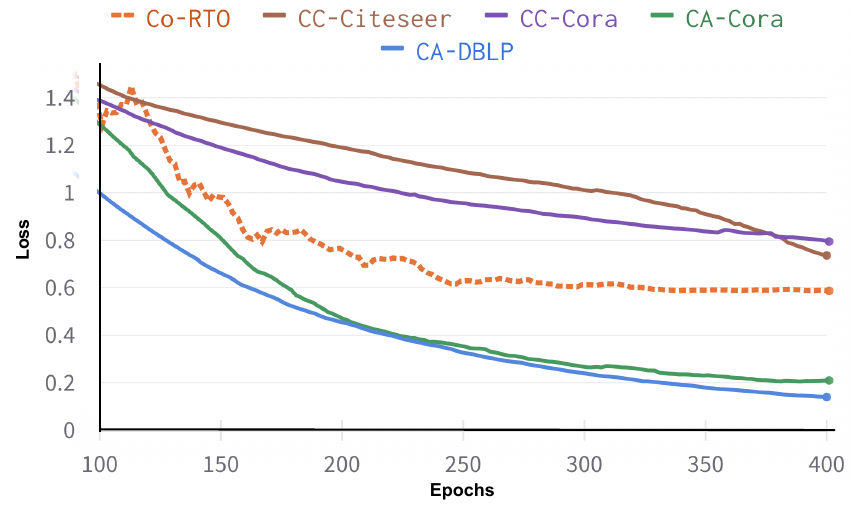}
\caption{Convergence analysis of DPHGNN over benchmark hypergraph and CO-RTO datasets.}
\label{fig: loss}
\vspace{-5mm}
\end{figure}

We also study the model's performance with increased message-passing blocks (i.e., the TAA and DFF blocks). The {DFF+HGNN layers} in Table \ref{table:5} indicate the increase in hypergraph message-passing layers; DFF stays constant throughout the layer increment. Moreover, the spectral inductive biases remain unchanged. However, we observe that with the increase in HGNN layers (4 to 8), the performance drops significantly ($5.1\%\downarrow$, $~30\%\downarrow$); over-smoothing of features \cite{c:40} remains the primary reason behind the performance drop with the increase in layers.

\subsubsection{Convergence Analysis.} Figure \ref{fig: loss} shows the significance of topological structure for the loss convergence in HGNN. We observe that the loss function tends to have a weak convergence for the co-citation networks compared to co-authorship networks.

\begin{table}[t]
\centering
\resizebox{0.8\columnwidth}{!}{
\begin{tabular}{lrrrr}
\hline
\textbf{Datasets} & \multicolumn{1}{l}{\textbf{Overall}} & \multicolumn{1}{l}{\textbf{w/o TAA}} & \multicolumn{1}{l}{\textbf{w/o SIB}} & \multicolumn{1}{l}{\textbf{w/o DFF}} \\ \hline
CA-DBLP           & 89.06                                & 87.72                                & 88.65                                & 86.54                                \\
CA-Cora           & 75.12                                & 73.6                                 & 74.52                                & 72.34                                \\
CC-Citeseer       & 65.77                                & 63.86                                & 59.47                                & 64.05                                \\
CC-Cora           & 67.51                                & 64.12                                & 61.15                                & 65.98                                \\
YelpRestaurant    & 35.82                                & 30.56                                & 34.69                                & 32.45                                \\
HouseCommitees    & 56.87                                & 52.86                                & 53.66                                & 55.68                                \\
Cooking200        & 55.74                                & 50.04                                & 49.53                                & 53.09                                \\
News20            & 81.07                                & 80.23                                & 76.55                                & 80.33                                \\ 
CO-RTO            & 51.39                                & 48.32                                & 46.75                                & 49.12                                \\\hline
\end{tabular}}
\caption{Ablation study to measure the impact of three blocks of DPHGNN -- Topology-Aware Attention (TAA), Spectral Inductive Biases (SIB), and Dynamic Feature Fusion (DFF). For benchmark datasets, we report mean accuracy (\%); since the CO-RTO dataset is imbalanced, we report macro F1-score. }
\label{table:4}
\vspace{-4mm}
\end{table}

\begin{table}[t]
\centering
\resizebox{0.9\columnwidth}{!}{
\begin{tabular}{l|ccc|ccc}
\hline
\textbf{Datasets} & \textbf{} & \textbf{TAA Layers} & \textbf{} & \textbf{} & \textbf{DFF Layers} &        \\ \cline{2-7} 
                  & 2-GCN     & 4-GCN               & 8-GCN     & 2-HGNN    & 4-HGNN              & 8-HGNN \\ \hline
CA-DBLP           & 89.06     & 88.34               & 86.95     & 89.06     & 88.92               & 46.03  \\
CA-Cora           & 75.12     & 74.85               & 73.8      & 75.12     & 69.01               & 32.54  \\
CO-Citeseer       & 65.77     & 64.98               & 63.94     & 65.77     & 59.82               & 25.86  \\
CO-Cora           & 67.51     & 66.02               & 64.22     & 67.51     & 58.25               & 28.05  \\
CoRTO             & 51.39     & 49.95               & 48.38     & 51.39     & 28.67               & 21.03  \\
YelpRestaurant    & 35.82     & 34.23               & 30.23     & 35.82     & 30.37               & 15.35  \\
HouseCommitees    & 56.87     & 55.01               & 51.95     & 56.87     & 50.4                & 20.91  \\
Cooking200        & 55.74     & 54.02               & 50.6      & 55.74     & 49.59               & 18.95  \\
News20            & 81.07     & 79.88               & 76.78     & 81.07     & 74.55               & 39.26  \\ \hline
\end{tabular}}
\caption{Ablation study to measure the impact of increment in message-passing layers over the TAA and DFF blocks.}
\label{table:5}
\vspace{-5mm}
\end{table}

\subsection{Performance on Hypergraph Isomorphism Test}
To validate our theoretical framework behind equivariant operator learning (EOL) within the DPHGNN model, we constructed two synthetic datasets; {\em "Isomorphic (Iso) HG" and "Non-Isomorphic (Non-Iso) HG"}. The classification results on these synthetic tasks are presented in Table \ref{tab:IsoHG}. The results demonstrate that DPHGNN achieves significant improvements, which are attributed to its robust spectral and spatial inductive biases.

\begin{table}[t]
\centering
\resizebox{0.6\columnwidth}{!}{
\begin{tabular}{l|ll}
\hline
Model    & Iso HG & Non-Iso HG \\ \hline
HGNN     & 55.23         & 53.25             \\
HGNN++   & 40.89         & 63.69             \\
HyperGCN & \textcolor{blue}{57.31}         & 58.25             \\
HNHN     & 38.85         & 77.3              \\
UniGCN   & 44.83         & 80.55             \\
UniSAGE  & 46.35         & \textcolor{blue}{85.00}             \\
UniGAT   & 48.84         & 84.21             \\
HENN     & 53.54         & 75.26             \\
HyGNN    & 51.63         & 80.64             \\ \hline
DPHGNN   & \textcolor{red}{73.38}         & \textcolor{red}{87.58}             \\ \hline
\end{tabular}}
\caption{Performance comparison on the Isomorphic and Non-Isomorphic datasets. The best performance is highlighted in \textcolor{red}{red}, and the best baseline is highlighted in \textcolor{blue}{blue}.}
\label{tab:IsoHG}
\vspace{-3mm}
\end{table}
\vspace{-3mm}
\section{The Return-to-Origin (RTO) Task}
This section describes modeling resource-constrained real-world data with hypergraph structure and reports the efficacy of DPHGNN for an industrial application -- the Return-to-Origin (RTO) task.

\paragraph{\bf Task Description.}
RTO refers to the situation where a product is returned to the point of origin instead of being delivered to the intended recipient \cite{c:30}. A product being RTO-ed substantially impacts global e-commerce and encourages a vicious cycle of financial losses for merchants, complex fraudulent transactions, etc. We model RTO as a semi-supervised hypernode classification task. In collaboration with an e-commerce giant, we deploy our proposed DPHGNN to address the RTO task.

\subsection{The CO-RTO Dataset}
Here, we describe the carefully curated CO-RTO (Co-order RTO) hypergraph structure. An RTO  involves various physical entities from an order transaction cycle (customer, product, supplier, and order) that can not be expressed using pairwise relations. This naturally motivates us to model higher-order interactions using hyperedges, containing entities as hypernodes. The RTO prediction task boils down to a binary hypernode classification problem -- predicting whether order-related nodes are RTOs. We obtained the CO-RTO dataset from our partner e-commerce company (Table \ref{table2}).

{\bf Hypergraph Construction.}
The CO-RTO dataset is an orderly transaction and entity relation hypergraph dataset. The raw data accompanies the feature information collected from partner e-commerce companies. Mainly comprised of four major entities that play key roles in a complete order transaction cycle namely, user, supplier, courier, and product. In this work, we have used an autoencoder for the conversion of heterogeneous features to homogeneous features to enable message flow. 
We formulate a hypergraph structure in the form of nodes that represent these entities with their corresponding feature vectors and family of hyperedges where each hyperedge comprises entity nodes from an order transaction flow.

\subsection{Experimental Results}
Table \ref{table:6} reports the performance of the competing models for the RTO task. Due to sparse connectivity, symmetrical structure, and class imbalance in CO-RTO, baselines suffer from over-smoothing of features over the message propagation. DPHGNN achieves $51\%$ macro-F1, outperforming other baselines by a significant margin (close to 7\%). We further report critical observations below. 

{\bf Symmetric and Sparse Incidence Structure.} In the CO-RTO hypergraph, the edge connectivity is sufficiently sparse (avg. hypernode density $|M|$ is $1.6$, hyperedge density $|\mu|$ is $0.82$; c.f. Table \ref{table2}),  constraining message passing between subsequent layers in HGNN. The fixed cardinality in the hyperedges of the CO-RTO hypergraph and the highly symmetrical structure tend to increase the multiplicity of eigenvalues, which is inherently linked with the Cheeger constant \cite{c:19} and the over-squashing phenomenon \cite{c:20}.  

{\bf Automorphism.} The hypergraph $HG = (V, \xi)$ of the CO-RTO dataset is a uniform hypergraph containing an automorphism group. More precisely, if $\pi$ is a permutation function over hypernode set $V$ of $HG$; $\pi: V \rightarrow V'$, with $|V| = |V'|$ and for each hyperedge $e \in \xi$ and $|e|=k$ (here $k=4$). Here $\pi$ satisfies $\pi(v)=v'$, where $v\in V$ and $v'\in V'$. Also $\forall e\in \xi \ \exists \ e' \in \xi$ such that $e' = \{\pi(v)|v\in e\}$. Now as every hyperedge $e_{i} \in \xi$, the image of $e_{i}$ under $\pi$ is also a hyperedge in $\xi$. For example: if $e_{i}=\{v_{i,1}, v_{i,2}, v_{i,3}, v_{i,4}\}$ then $\pi(e_{i})=\{\pi(v_{i,1}), \pi(v_{i,2}), \pi(v_{i,3}), \pi(v_{i,4})\}$ is also in $\xi$. Given the general spatial HGNN message-passing framework, automorphism in hypergraph preserves the incidence relations between vertices $u$ and edge $e$ as $HG_{u, e} = HG_{\pi(u), \pi(e)}$. Given two hypernodes, ($a, b$) under automorphism $\pi$ are $a = \phi(b)$, their feature update equations \ref{eq:auto1} and \ref{eq:auto2} are:
\begin{equation}\small\label{eq:auto1}
    \hat{x}_a^{(l)} = \text{Update}^{(l)}\left(x_a^{(l)}, \text{Aggr}^{(l)}\left(\left\{\left\{\text{Aggr}^{(l-1)}\left(x_u^{(l-1)}\right) \forall u \in e\right\} \forall e \in E_a\right\}\right)\right)
\end{equation}

\begin{multline}\small\label{eq:auto2}
    \hat{x}_b^{(l)} = \operatorname{Update}^{(l)}\bigg( x_b^{(l)}, \operatorname{Aggr}^{(l)}\Big( \{ \{ \operatorname{Aggr}^{(l-1)}(x_{\pi(u)}^{(l-1)}) \\
    \forall \pi(u) \in \pi(e)\} \forall \pi(e) \in E_b \} \Big) \bigg)
\end{multline}

Hence given the automorphism $\pi$, feature update leads to $\hat{x}_a^{(l)} = \hat{x}_b^{(l)}$. Therefore, the spatial HGNNs bounded by 1-GWL expressivity perform sub-optimally due to symmetry-induced automorphism groups. This constitutes a major part of our novel contributions. To our knowledge, our method (DPHGNN) is the first HGNN that breaks the automorphism groups and is 3-GWL expressive.

{\bf Over-smoothing.} We perform experiments with dual-layer HGNNs. Still, the features get smoothed over message passing due to the sparse incidence structure. This is the main reason behind the suboptimal performance of baselines on CO-RTO. Along with propagating messages with skip connections \cite{c:21}, DPHGNN leverages static inductive priors to prevent over-smoothing of relevant feature information. 

{\bf Class Imbalance.} In CO-RTO, class imbalance occurs naturally (a ratio of $6568$:$26827$ for RTO:Non-RTO), and HGNN tends to overfit the majority class. 
These observations lead to a drastic performance drop with the hypergraph-based MPNN models.

\begin{table}[t]
\centering
\resizebox{0.7\columnwidth}{!}{
\begin{tabular}{l|lll}
\hline
{\bf Method} & {\bf Mean Acc}           & {\bf Macro F1}           & {\bf Micro F1}           \\ \hline
HGNN     & 80.51±1.7          & 45.01±0.6          & 80.51±1.6          \\
HGNN+    & 80.47±0.9          & 44.74±1.5          & 80.47±0.9          \\
HyperGCN & 80.50±0.6          & 44.59±0.2          & 80.49±0.6          \\
HNHN     & 80.49±0.8          & 44.60±0.2          & 80.49±0.8          \\
UniGCN   & 80.50±2.6          & 44.86±3.2          & 80.49±2.8          \\
UniSAGE  & \textcolor{blue}{81.66±2.9} &\textcolor{blue}{46.66±3.1} & \textcolor{blue}{81.66±2.9} \\
UniGAT   & 80.48±0.6          & 44.70±0.4          & 80.48±0.6          \\ \hdashline
DPHGNN  & \textcolor{red}{86.02±1.2} & \textcolor{red}{51.39±0.8} & \textcolor{red}{86.02±1.2} \\
\hline
\end{tabular}}
\caption{Performance comparison on the CO-RTO dataset. The best performance is highlighted in \textcolor{red}{red}, and the best baseline is highlighted in \textcolor{blue}{blue}.}
\label{table:6}
\vspace{-5mm}
\end{table}

\vspace{-2.6mm}
\subsection{Deployment Workflow \label{sec:deployment}}
This section presents a brief description of the deployment of DPHGNN in a large-scale e-commerce platform with a transductive setting \cite{c:37}, Figure \ref{fig: deployment} in supplementary illustrates the deployment pipelines.

{\bf Data Pipelines.}
The large-scale e-commerce transaction data is pulled through the Apache Spark streaming service \cite{c:43} and distributed across multiple clusters to compute the homogeneous hypernode features and specify user, product, supplier, and courier entities inside a co-order hyperedge. These node features and edges are hosted in a feature store infrastructure built on top of a data lake and partitioned for efficient read and write operations. 

{\bf Test-Train Hypergraph creation.}
Based on this model’s inference, some of the orders get actioned. Therefore, the possible outcome of RTO remains unknown due to the model’s action, and hence the data is not useful for future model training. To resolve this, a holdout sample is maintained where no actions are taken and are used as training data for transductive training. 

{\bf Model Building and Inference.}
In the current transductive setting, the target variable is only available for the past holdout transactions; the model is trained only on those hypernodes and hyperedges using training masks. All the edges and nodes go through the forward pass. Once the model is trained, the predictions for the latest transactions are collected and stored. We used NVIDIA T4 GPU clusters with 20 workers for network training and inference. The orders flagged are published in a Kafka queue \cite{c:44} to be consumed and blocked by the order management system.

\section{Conclusion}
In this paper, we proposed DPHGNN, a novel hypergraph message-passing approach for representation learning. It explicitly learns static and dynamic feature representations from multiple views of lower-order graph topology. DPHGNN outperformed seven SOTA baselines over eight benchmark datasets. We also showed its expressive power theoretically. DPHGNN was also tested on a new e-commerce dataset (CO-RTO) and deployed by our industry partner for RTO prediction.

\section*{Acknowledgement}
The work was funded by Meesho.


\bibliographystyle{ACM-Reference-Format}
\bibliography{KDD24}

\appendix
\newpage

\section{Supplementary Material}

\subsection{Proof of Proposition \label{proof:p2}}
\begin{proof}
    We generalize the work of \cite{c:16, c:17} on hypergraph color refinement from (see Eq. \eqref{eq:5} in the main text) and using tensors to encode the color, for hypergraphs having ${k}$ vertices inside an edge $H C \in \Sigma^{n^k}$, as the color of a node is induced by the coloring of its edges by a $\operatorname{map} f_H(e): V^k \rightarrow \Sigma$. Therefore, the tensor $HC \in \mathbb{R}^{n^k \times a}$ encodes $ {k}$ vertices inside an edge $i \in[n]^k$ with $H C_i \in \Sigma$ and color $\Sigma \in \mathbb{R}^a$.
    Initialising $H C_0^H(v)=1$ for all $ v \in V(H)$, the color refinement is given by:\\ 
    $H C_{t+1}^H(v)=\text{HASH}\{\{\{{H C_t^H(u) \mid u \in f_H(e)}\} \mid e \in E_{i} \ \text{with}\ v \in f_H(e)\}\}$, $\{\{.\}\}$ represents multiset notation.\ 
    $H_{1}$ and $H_{2}$ are non-isomorphic if \\
    \begin{align*}
    \left\{\left\{H C_i^{H_1}(v) \mid v \in V\left(H_1\right)\right\}\right\} \neq\left\{\left\{H C_i^{H_2}(v) \mid v \in V\left(H_2\right)\right\}\right\}
    \end{align*}
    
    We start by proving for any HGNN of the form of Eq\ref{eq:1}, there exists an injective mapping $\Gamma$ between every iteration of color refinement and the layer of message propagation such that $\Gamma^{(t, l)}\left(H C_v^t\right) \rightarrow\left(h_v^l\right)$. 
    
    Assume that injective mapping $\Gamma^{(t, l)}$ exists. For $(t, l)=0$ trivially, both tensors are initially assigned some value under injectivity $H C_v^0, h_v^0$.
    
    For $(t+1, l+1)$, the feature update is given by,
    \begin{align*}
    h_v^{l+1}=\phi_3\left(h_{v^{\prime}}^l, \phi_2\left\{\left\{\phi_1\left(\left\{\left\{h_u^l \mid u \in e\right\}\right\}\right) \mid e \in E_i\right\}\right\}\right)
    \end{align*}
    as $\phi_1, \phi_2, \phi_3$ are all injective, they could be composed into an injective function $\varphi$, therefore, \\
    \noindent
    \begin{align*}
        &= \varphi\left(h_v^l,\left\{\left\{\left(\left\{\left\{h_u^l \mid u \in e\right\}\right\}\right) \mid e \in E_i\right\}\right\}\right) \\
        &= \varphi\left(\Gamma^{(t, l)}\left(H C_v^t\right),\left\{\left\{\left(\left\{\left\{\Gamma^{(t, l)}\left(H C_u^t\right) \mid u \in e\right\}\right\}\right) \mid e \in E_i\right\}\right\}\right) \\
        &= \Gamma^{(t, l+1)}\left(H C_{v^{\prime}}^t,\left\{\left\{\left(\left\{\left\{\left(H C_u^t\right) \mid u \in e\right\}\right\}\right) \mid e \in E_i\right\}\right\}\right) \\
        &= \Gamma^{(t+1, l+1)}\left(H C_v^{t+1}\right)
    \end{align*}
    Where injective function $\Gamma^{(t+1, l+1)}$ is induced by $\varphi, \Gamma^{(t, l)}$.
    Hence by principles of induction, we can say, there exists an injective function $\Gamma^{(t, l)}$ at time $t$ and layer $l$ such that $\Gamma^{(t, l)}\left(H C_v^t\right) \rightarrow\left(h_v^l\right)$. 
    
    Therefore, at iteration $T$, if 
    $
    HC_{(t)}\left(V_1, H_1\right) \neq H C_{(t)}\left(V_2, H_2\right)$, \text{then}
    $\left\{\left\{h_{v, H_1}^T\right\}\right\}_{v \in V_1} \neq\left\{\left\{h_{v, H_2}^T\right\}\right\}_{v \in V_2}$.
\end{proof}

\subsection{Proof of Theorem 1 \label{proof:t1}}

\begin{proof}
    \textbf{Part 1:} We first establish that a function $g_{\theta} \in F^{\text{DPHGNN}}$ belongs to a family of $k$-order invariant networks. 
    
    Let $S_{n}$ denote the permutation group on $n$ elements. In the dynamic feature fusion (DFF), dynamic features $m^{(l)}_{p} = \phi_{mask}[\text{L}]$ where $p \in V_{*}$ comes from the update rule: \\ $L = \left[\operatorname{Aggr}^{(l)}\left(\left\{m_u^{(l-1)} \forall u \in N(p)\right\}\right)\right]$. The function $\phi_{mask}$, is bijective; therefore, the updated representations satisfy $\text{L}(f \cdot{m^{(l)}_{p}}) = f \cdot L(m^{(l)}_{p})$, $\text{f} \in S_{n}$. Hence, the update mechanism of DPHGNN follows the form (a $k$-order invariant graph network) $F = m \circ h \circ L_{d} \circ \sigma \circ \dots \circ \sigma \circ L_{1}$, where $h$ is an invariant linear layer, $\sigma$ is non-linear activation function and $L_{i}$ are equivariant layers, $m$ is the multilayer perceptron (FC). \citet{c:17}, in Theorem 1, proved the expressive power of $k$-order invariant graph network equivalent to the $k$-WL test (for graphs).
    
    \textbf{Part 2:} We quantify the expressive power of DPHGNN as a 1-FWL test. The main difference is between the update rules of WL and FWL tests. FWL assigns distinct colorings to the $k$-tuple of nodes. The feature fusion mechanism in DPHGNN updates the hypernode aggregation as \vspace{0.5mm}\\ 
    \begin{align*}
        m_e^{(l-1)}&=\operatorname{Aggr}^{(l-1)}\left(\{\{m_u^{(l-1)} \forall u \in e\}\} \cup\{\{m_p^{(l)} \forall p \in E\}\}\right).
    \end{align*}
    Therefore, now the tensor $\text{HC} \in \mathbb{R}^{(n^k \times a) + b}$ encodes the hypernodes, where coloring $b$ is assigned by $\operatorname{map} f_G(e): V \rightarrow V^{'}$ and $V^{'} \in \mathbb{R}^{b}$. We formulate the FWL version of hypergraph color refinement as, \\ \vspace{1mm}
    $\text{HC}_{t+1}^H(v)=\text{HASH}\{\{\{{H C_t^H(u) \mid u \in \psi_H(e)}\} \mid e \in E_{i}  \text{with}\ v \in \psi_H(e)\}\}$, where $\operatorname{map} \ \psi_H(e) := (f_G \circ f_H)(e)$.
    The update mechanism of the hypernodes in DPHGNN is given by: \vspace{0.5mm}
    \\
    $x_v^{(l)}=\operatorname{Update}^{(l)}(\{x_v^{(l-1)}, \operatorname{Aggr}^{(l)}\{m_e^{(l-1)} \forall e \in E_i\}\})$ \vspace{0.5mm}. 
    The injectivity of Aggregation and Readout functions follows similarly from Proposition \label{proposition1}. The expressive power of $k$-FWL test is equivalent to $(k+1)$-WL for $k\geq2$. The function $g_{\theta} \in F^{\text{DPHGNN}}$ is at least as powerful as the 3-GWL test.
\end{proof}

\begin{figure}[!t]
\centering
\includegraphics[width=1.0\columnwidth]{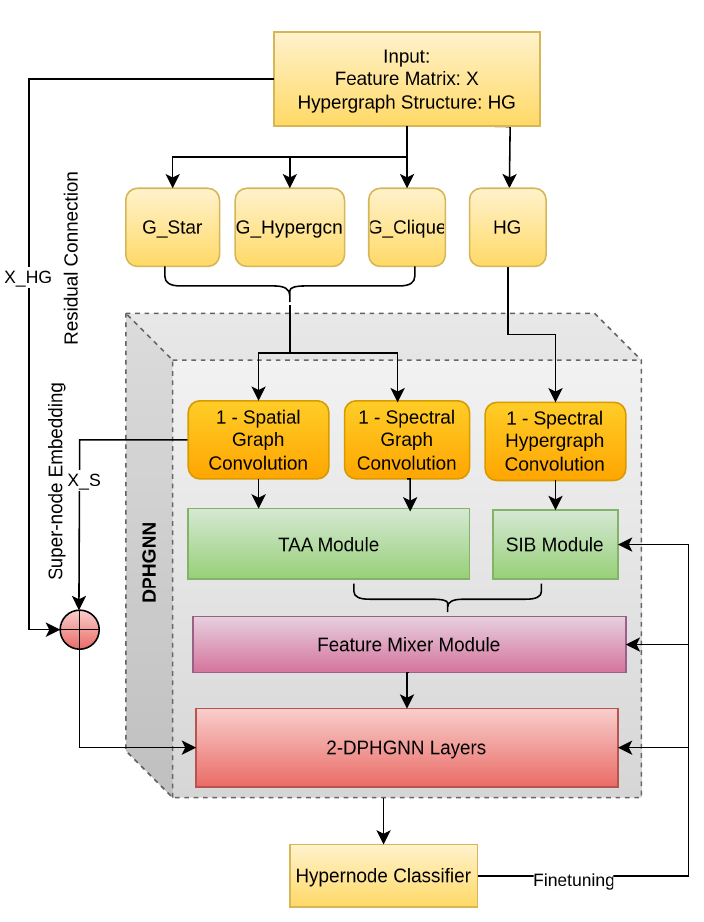}
\caption{Generalized version of experimentation pipeline.}\label{fig:training_pipeline}
\vspace{-5mm}
\end{figure}

\begin{figure*}
\centering
\includegraphics[width=0.85\textwidth]{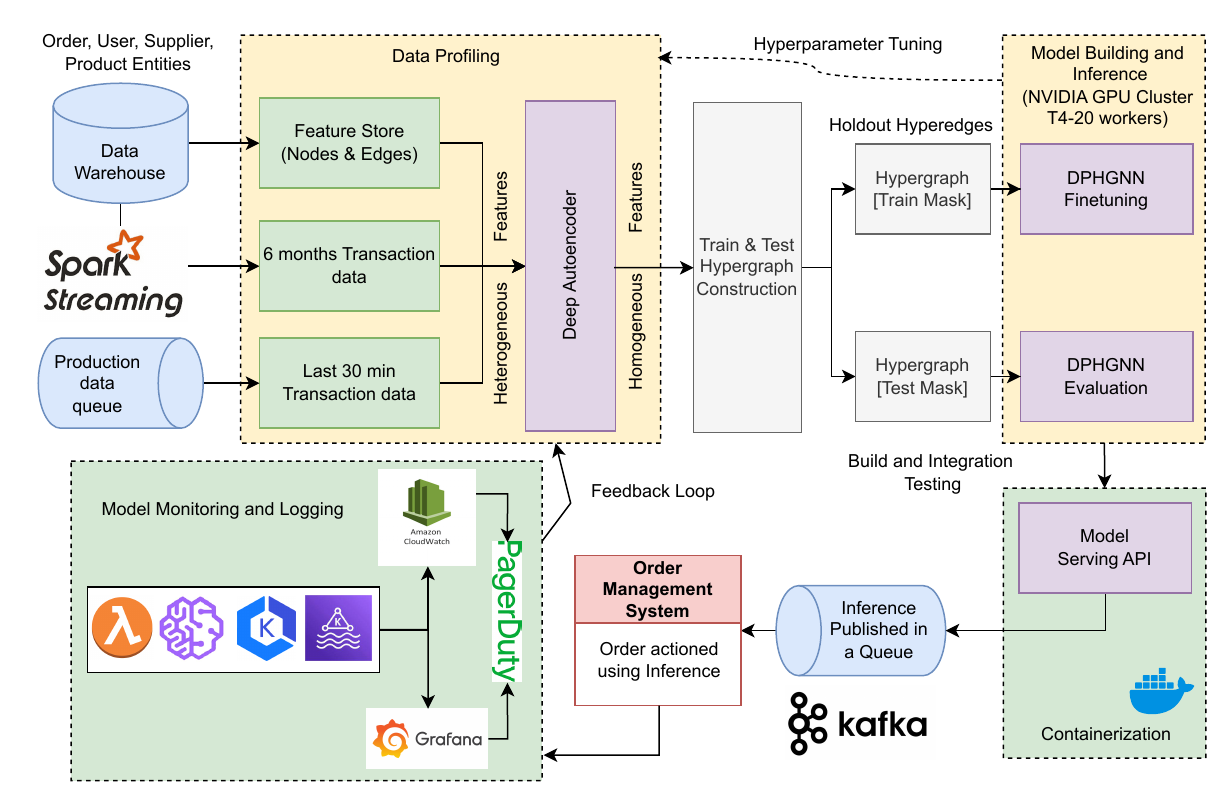}
\caption{Deployment details of DPHGNN into our partner e-commerce ecosystem.}
\label{fig: deployment}
\vspace{-3mm}
\end{figure*}

\begin{table*}[t]
\centering
\resizebox{2.0\columnwidth}{!}{
\begin{tabular}{l|llllllll}
\hline
Dataset          & HGNN  & HGNN+ & HyperGCN & HNHN & UniGCN & UniSAGE & UniGAT & DPHGNN \\ \hline
CA-DBLP          & 605.9 & 401   & 563.4    & 358  & 432    & 455     & 502    & 528    \\
CA-Cora          & 228.4 & 145   & 183.4    & 122  & 161    & 157     & 174    & 182    \\
CC-CiteSeer      & 290.1 & 122   & 171.1    & 135  & 140    & 152     & 186    & 168    \\
CC-Cora          & 301.4 & 167   & 205.4    & 130  & 174    & 193     & 201    & 220    \\
Yelp Restaurant  & 570   & 194   & 495      & 185  & 218    & 223     & 247    & 312.2  \\
House Committees & 192   & 114   & 153      & 66   & 72     & 96      & 150    & 168    \\
Cooking200       & 502   & 156   & 258      & 144  & 174    & 216     & 216    & 233    \\
News20           & 406   & 138   & 210      & 108  & 102    & 108     & 156    & 174    \\
CO-RTO           & 384   & 192   & 306      & 108  & 138    & 150     & 174    & 205    \\ \hline
\end{tabular}}
\caption{Runtime analysis between proposed DPHGNN with baseline architectures (in seconds).}
\label{tab:runtime}
\end{table*}

\subsubsection{Convergence Analysis.}
Here, we describe the convergence analysis of DPHGNN for the standard hypergraph datasets and the CO-RTO dataset. Figure \ref{fig: loss} reflects the behavior of cross-entropy loss over the change in hypergraph topology. The convergence for co-citation is slow as the spectral information is propagated as inductive priors. However, DPHGNN achieves a fast convergence rate for the Co-authorship datasets as explicit lower-order information is propagated as dynamic feature fusion. To learn in resource-constrained settings for the Co-RTO dataset, the loss initially increases and converges better than in co-citation networks. This signifies the flow of inductive priors. Due to the sparse and highly symmetrical structure, the initial performance suffers from providing better representations, the static and dynamic fusion of lower-order information helps message propagation to achieve better convergence.

\subsection{Implementation Details}
DPHGNN deploys an end-to-end differentiable training pipeline to learn feature representations of both graph and hypergraph structures. Figure \ref{fig: deployment} illustrates data flow in DPHGNN architecture, given feature matrix $X$ and hypergraph $HG$. The hypergraph topology is decomposed into three different graph topologies, which provide input for topology-aware attention (TAA), and $X$ is provided to static inductive biases (SIB) in parallel. At last, the feature fusion module performs HGNN aggregation and updates of explicit graph representations fused with hypernode aggregations. The loss is backpropagated from the prediction layer to all the sub-modules of DPHGNN.

\paragraph{Hyperparameter Details.}
Table \ref{table7} summarizes the hyperparameter details for each submodule of DPHGNN architecture. We used GridSearchCV (from the sklearn library) for hyperparameter tuning. The sub-modules of DPHGNN, namely (i) the GNN module comprises the graph convolution training of multiple views of decomposed graph topology (as described in the TAA section in the main paper); (ii) the TAA module contains a multi-head wrapper of cross-attention modules with specific heads for each of the hypergraph datasets; (iii) the SIB module is trained as a hybrid spectral convolution layer. For the model parameter optimizations, we use a specific parameter group for each of the submodules of the DPHGNN architecture.

\paragraph{Experimentation Details.}
For the semi-supervised hypernode classification task, we train the model on each dataset for 400 epochs and infer the trained model from the last epoch. We perform 50 iterations of DPHGNN on each dataset, with five train/test splits and ten different random seeds. Figure \ref{fig:training_pipeline} presents a concise view of training mechanism adopted by DPHGNN.

\begin{table*}[]
\centering
\resizebox{\textwidth}{!}{
\begin{tabular}{c|c|cccccc}
\hline
Datasets    & Modules    & lr    & weight\_decay & dropout & attention\_heads & hidden & \# layers \\ \hline
            & GNN module & 0.1   & 5.0E-04       & 0.2     & -                & 64     & 2         \\
CA-Cora     & TAA module & 0.001 & 0.001         & 0.5     & 2                & 32     & 1         \\
            & SIB module & 0.01  & 0.0005        & 0.4     & -                & 64     & 1         \\
            & DFF module & 0.01  & 5.0E-04       & 0.5     & -                & 64     & 2         \\ \hline
            & GNN module & 0.1   & 5.0E-05       & 0.6     & -                & 64     & 2         \\
CA-DBLP     & TAA module & 0.1   & 0.001         & 0.5     & 2                & 32     & 1         \\
            & SIB module & 0.01  & 0.0005        & 0.4     & -                & 64     & 1         \\
            & DFF module & 0.01  & 5.0E-04       & 0.6     & -                & 64     & 2         \\ \hline
            & GNN module & 0.001 & 0.001         & 0.5     & -                & 64     & 2         \\
CC-Cora     & TAA module & 0.01  & 5.0E-04       & 0.5     & 8                & 8      & 1         \\
            & SIB module & 0.01  & 0.0005        & 0.4     &                  & 64     & 1         \\
            & DFF module & 0.01  & 5.0E-04       & 0.6     & -                & 64     & 2         \\ \hline
            & GNN module & 0.1   & 1.0E-05       & 0.9     & -                & 64     & 2         \\
CC-Citeseer & TAA module & 0.001 & 0.001         & 0.5     & 8                & 8      & 1         \\
            & SIB module & 0.01  & 0.0005        & 0.4     & -                & 64     & 1         \\
            & DFF module & 0.01  & 5.0E-04       & 0.6     & -                & 64     & 2         \\ \hline
            & GNN module & 0.01  & 5.0E-05       & 0.3     & -                & 64     & 2         \\
CO-RTO      & TAA module & 0.001 & 0.001         & 0.5     & 2                & 32     & 1         \\
            & SIB module & 0.01  & 0.0005        & 0.4     & -                & 64     & 1         \\
            & DFF module & 0.01  & 5.0E-04       & 0.5     & -                & 64     & 2         \\ \hline
            & GNN module & 0.01  & 5.0E-05       & 0.3     & -                & 64     & 2         \\
YelpRestaurant      & TAA module & 0.001 & 0.001         & 0.5     & 2                & 64     & 1         \\
            & SIB module & 0.01  & 0.0005        & 0.4     & -                & 64     & 1         \\
            & DFF module & 0.01  & 5.0E-04       & 0.5     & -                & 64     & 2         \\ \hline
            & GNN module & 0.01  & 5.0E-05       & 0.3     & -                & 64     & 2         \\
HouseCommittees      & TAA module & 0.001 & 0.001         & 0.5     & 4                & 32     & 4         \\
            & SIB module & 0.01  & 0.0005        & 0.4     & -                & 64     & 1         \\
            & DFF module & 0.01  & 5.0E-04       & 0.5     & -                & 64     & 2         \\ \hline
            & GNN module & 0.01  & 5.0E-05       & 0.3     & -                & 64     & 2         \\
Cooking200      & TAA module & 0.001 & 0.001         & 0.5     & 8                & 64     & 1         \\
            & SIB module & 0.01  & 0.0005        & 0.4     & -                & 64     & 1         \\
            & DFF module & 0.01  & 1.0E-03       & 0.5     & -                & 64     & 2         \\ \hline
            & GNN module & 0.01  & 1.0E-2       & 0.3     & -                & 64     & 2         \\
News20      & TAA module & 0.001 & 0.001         & 0.5     & 4                & 32     & 1         \\
            & SIB module & 0.01  & 0.0005        & 0.4     & -                & 64     & 1         \\
            & DFF module & 0.01  & 0.01       & 0.5     & -                & 64     & 2         \\ \hline

\end{tabular}}
\caption{Hyperparameters used in our experiments.}
\label{table7}
\end{table*}

\subsection{Embedding Visualization \label{sec:viz}}
DPHGNN performs a series of operations to leverage both spatial and spectral inductive biases. Here, we provide a comparative analysis of change in embeddings between DPHGNN and the best baseline model respective to the hypergraph datasets in Figures \ref{fig:CORTO}, \ref{fig:CACORA} and \ref{fig:NEWS20}.

\begin{figure*}[!t]
    \centering
    \begin{minipage}{0.70\textwidth}
        \includegraphics[width=0.33\textwidth]{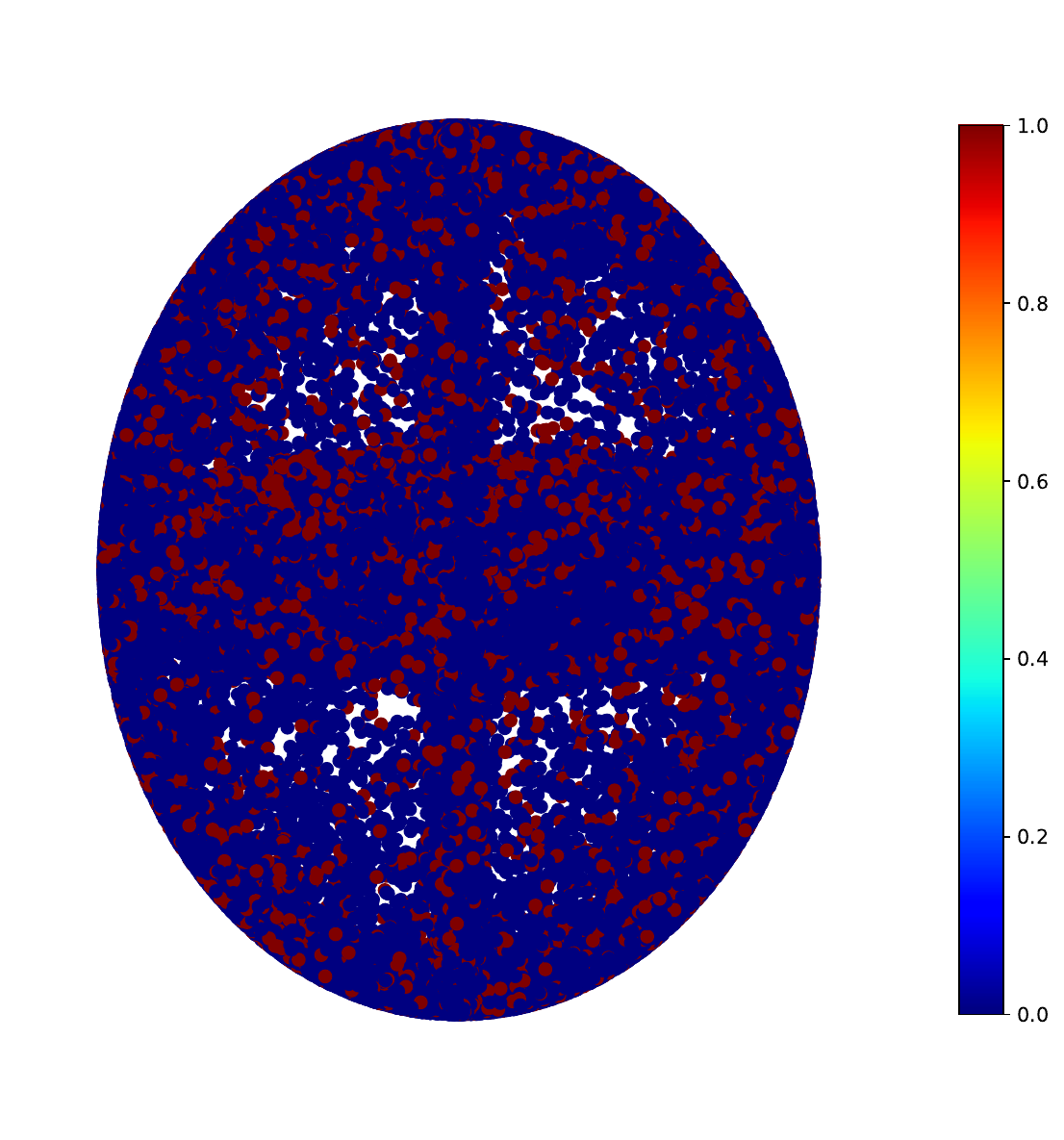}\hfill
        \includegraphics[width=0.33\textwidth]{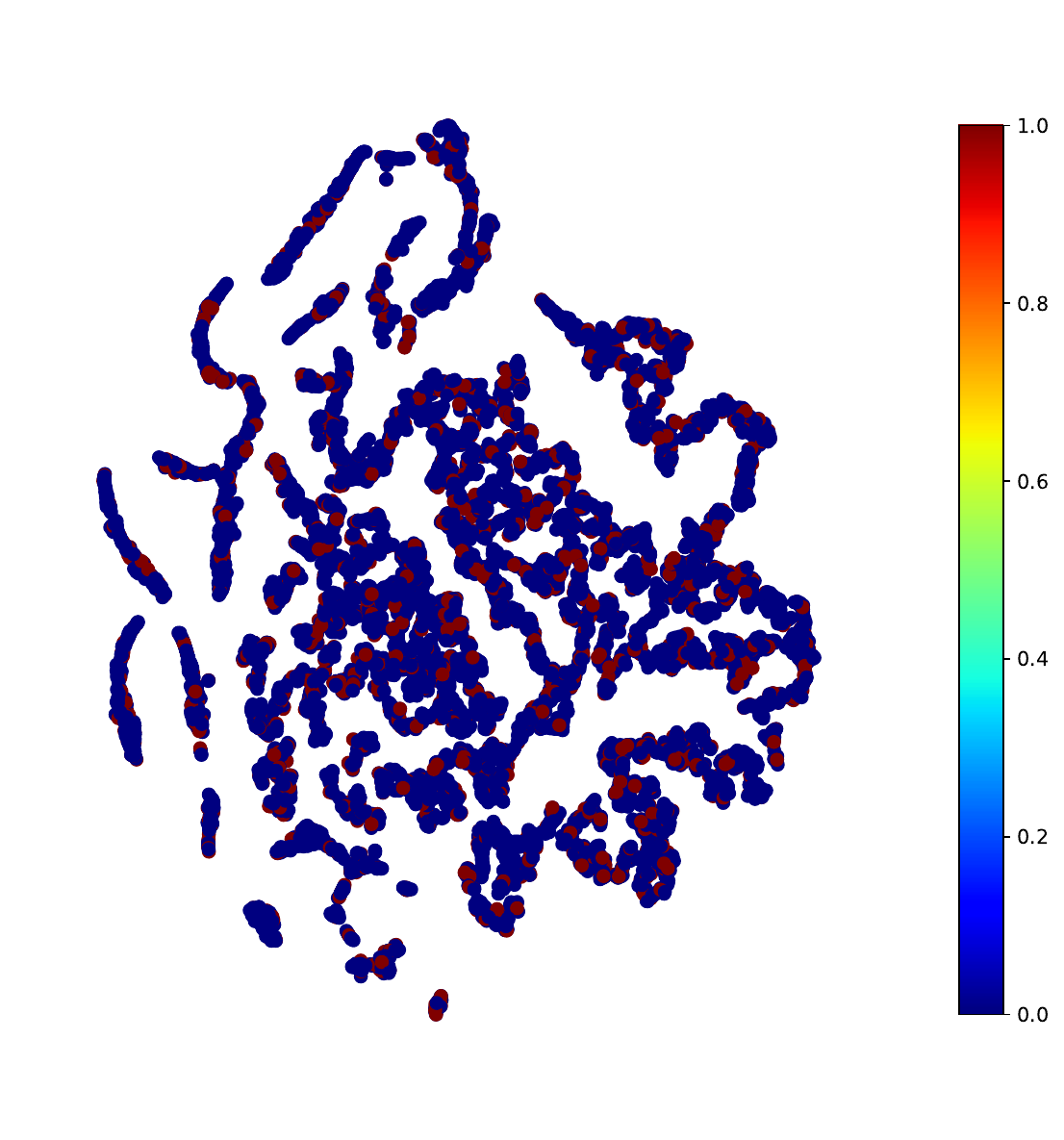}\hfill
        \includegraphics[width=0.33\textwidth]{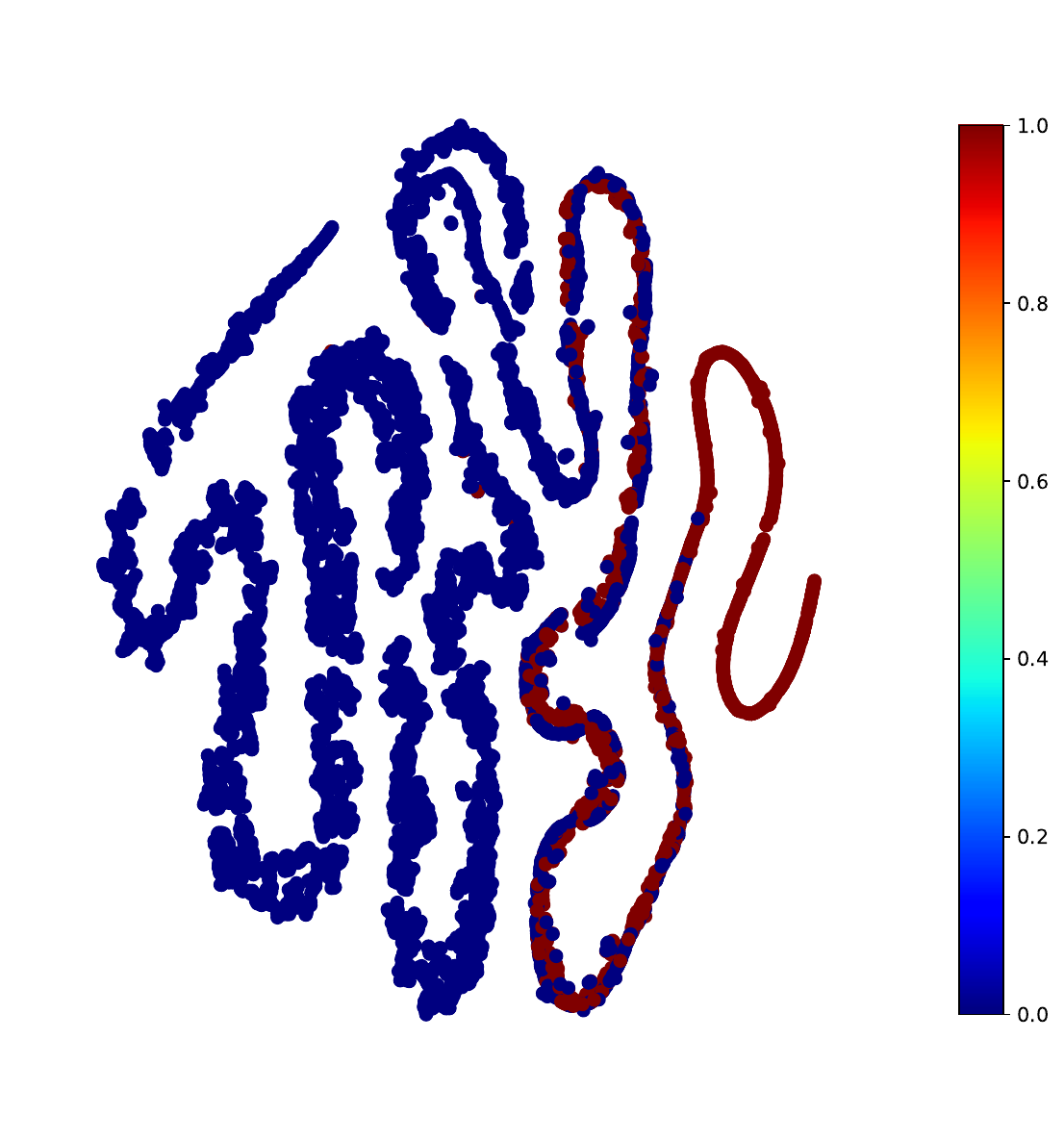}
    \end{minipage}%
    \begin{minipage}{0.20\textwidth}
        \caption{Visualization of feature embedding update on the CO-RTO dataset (2 classes) -- initial embedding (left), UniGCN embedding update (middle), DPHGNN embedding update (right)}
    \label{fig:CORTO}
    \end{minipage}
\end{figure*}

\begin{figure*}[!t]
    \centering
    \begin{minipage}{0.70\textwidth}
        \includegraphics[width=0.33\textwidth]{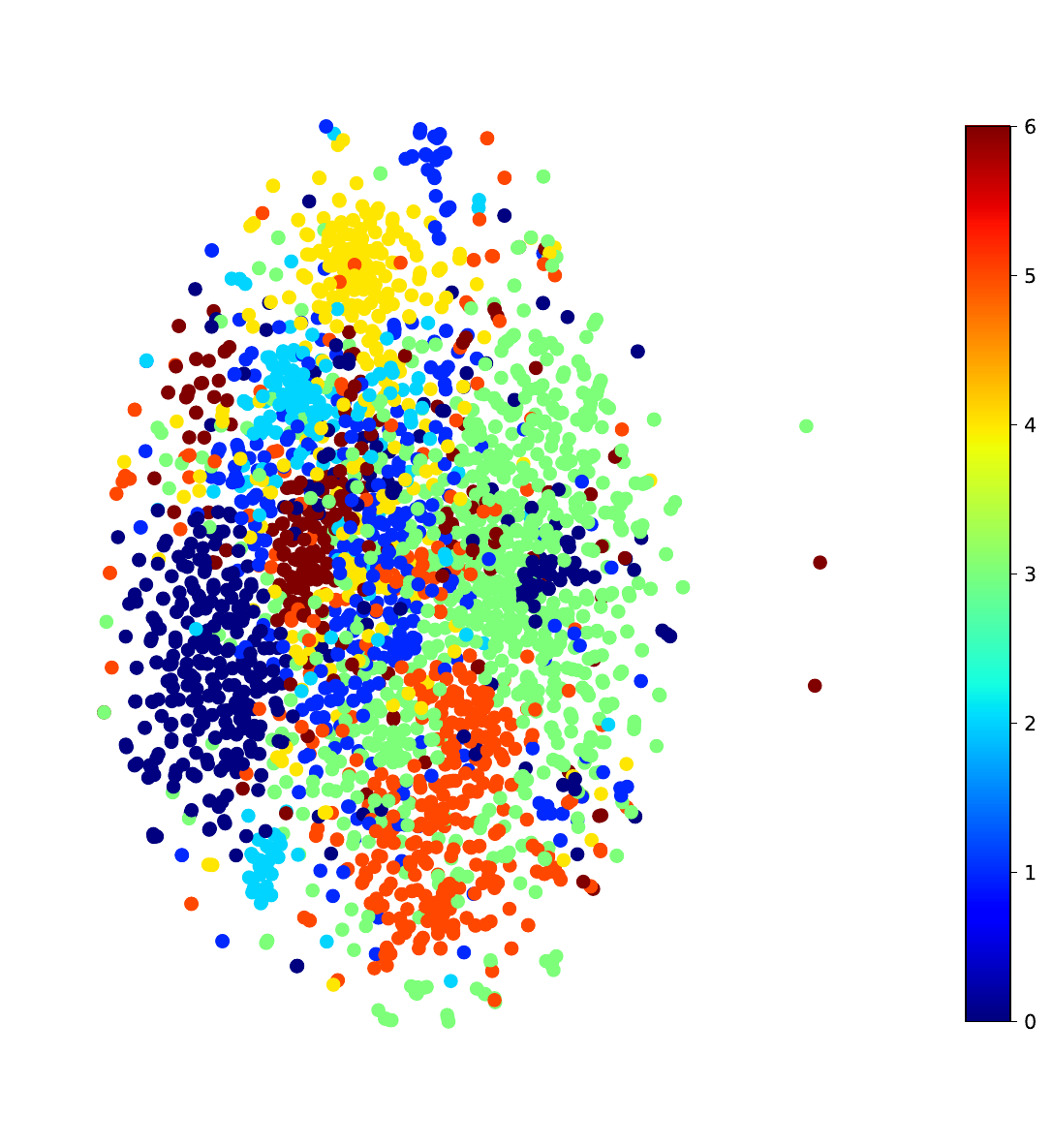}\hfill
        \includegraphics[width=0.33\textwidth]{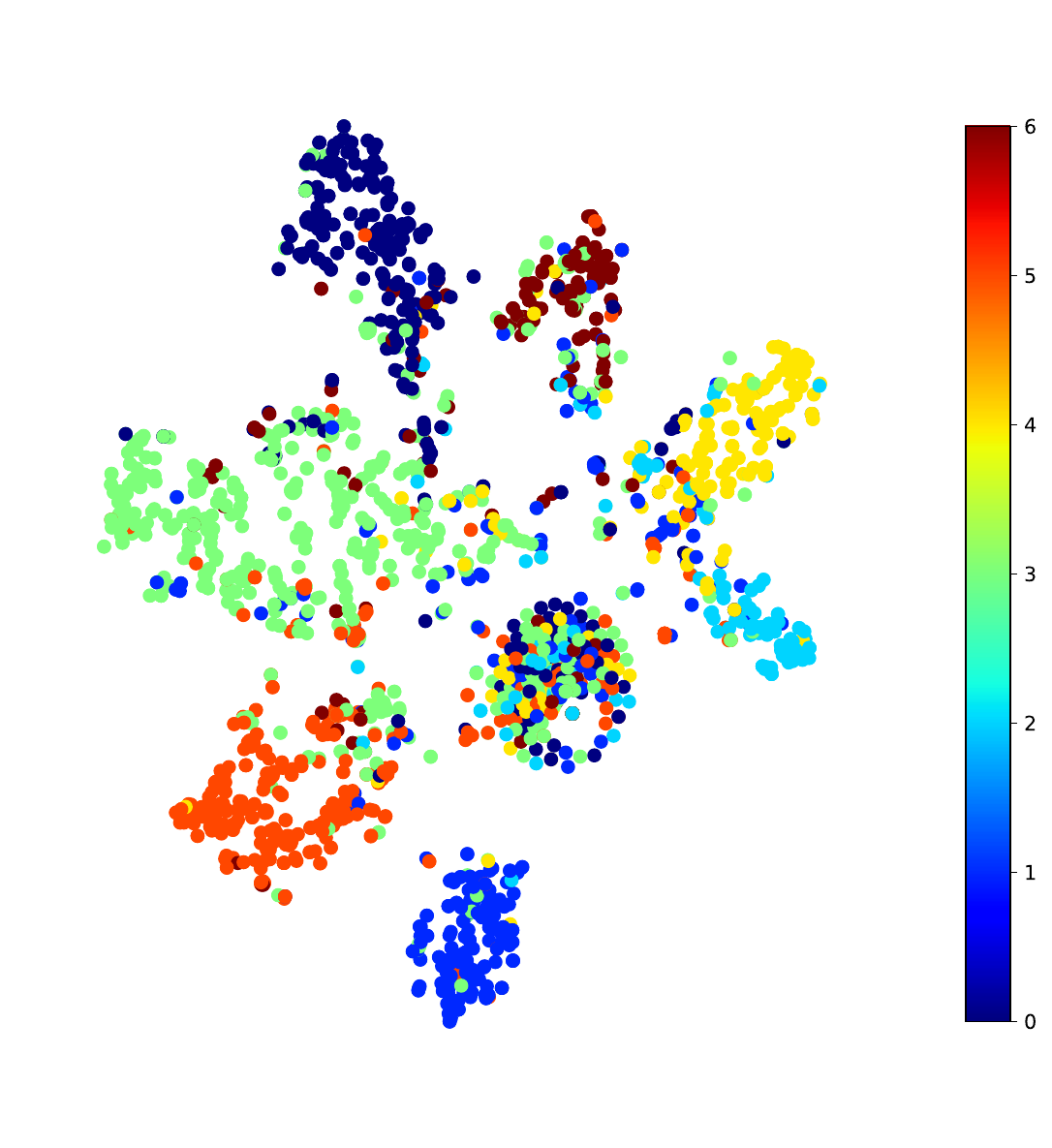}\hfill
        \includegraphics[width=0.33\textwidth]{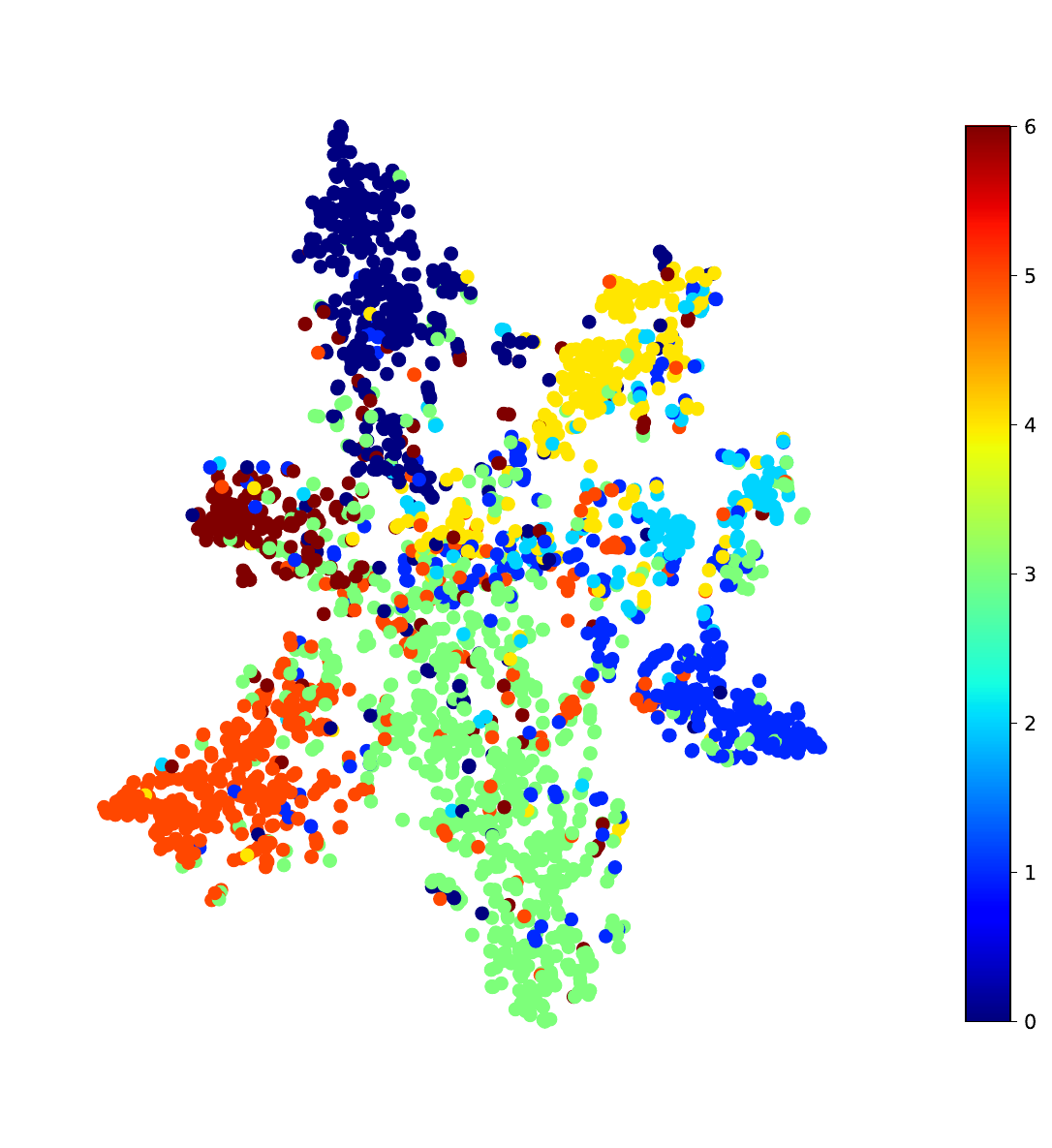}
    \end{minipage}%
    \begin{minipage}{0.20\textwidth}
        \caption{Visualization of feature embedding update on the CA-Cora dataset (5 classes) -- initial embedding (left), UniGCN embedding update (middle), DPHGNN embedding update (right).}
    \label{fig:CACORA}
    \end{minipage}
\end{figure*}

\begin{figure*}[!t]
    \centering
    \begin{minipage}{0.70\textwidth}
        \includegraphics[width=0.33\textwidth]{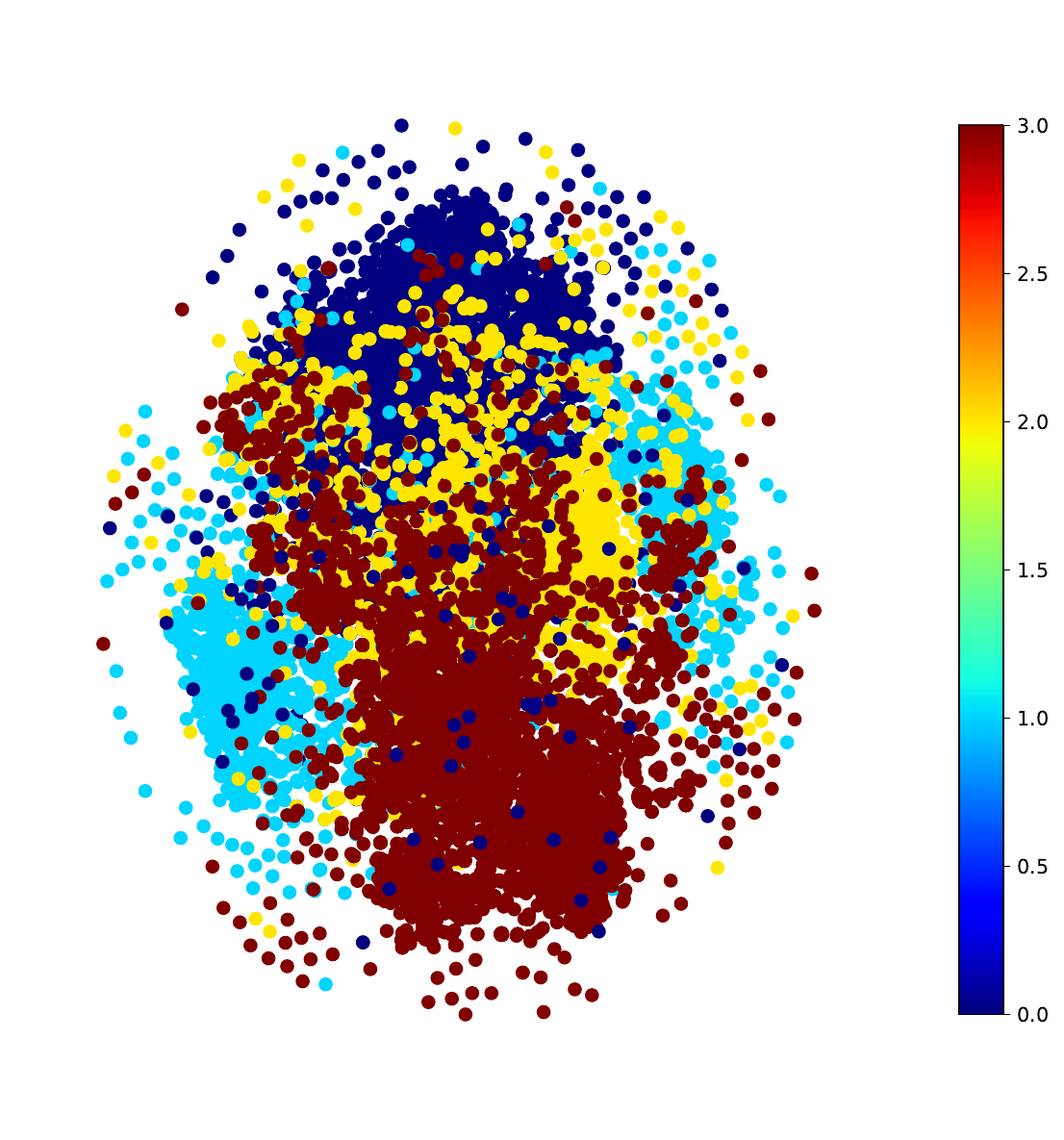}\hfill
        \includegraphics[width=0.33\textwidth]{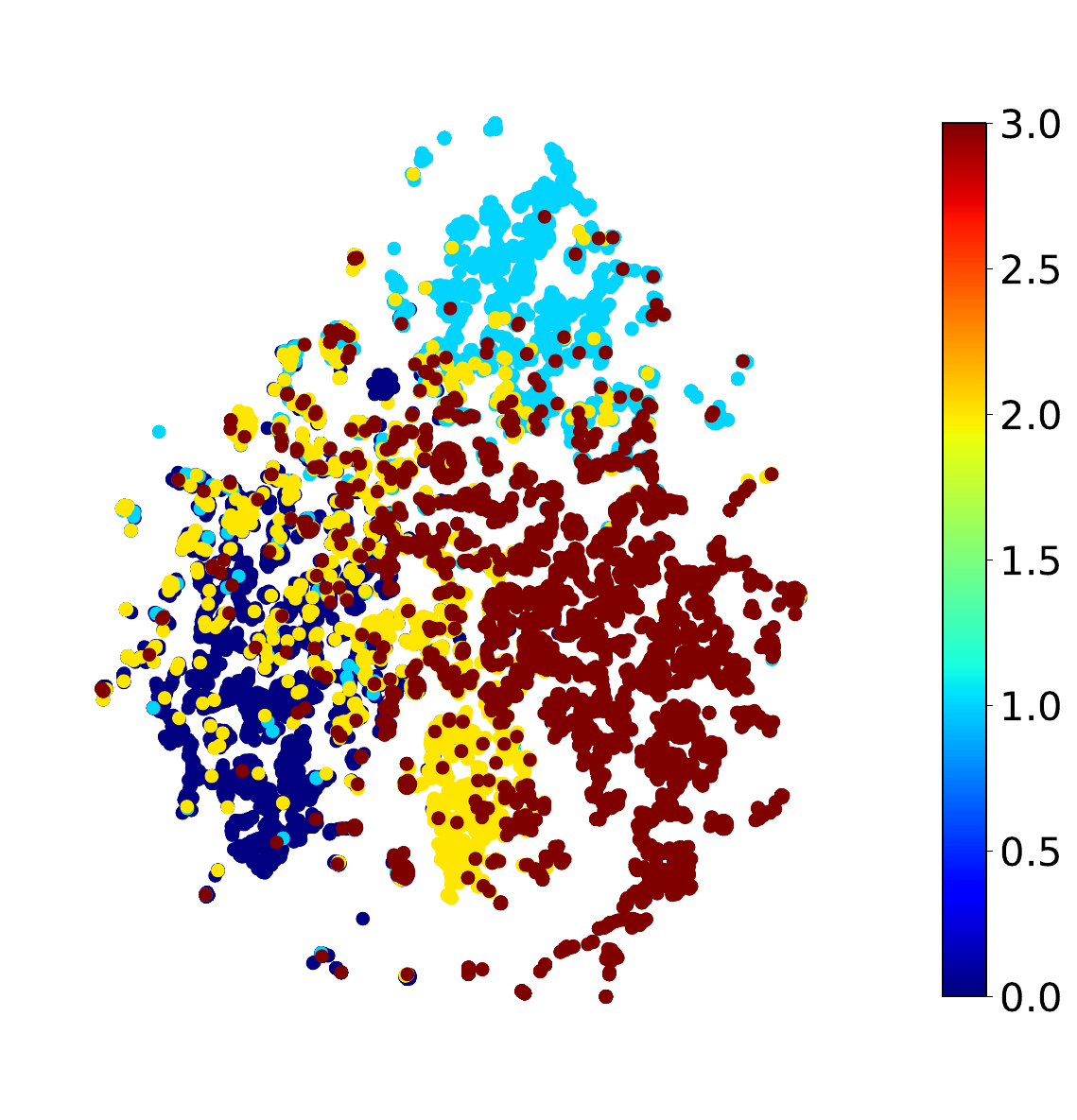}\hfill
        \includegraphics[width=0.33\textwidth]{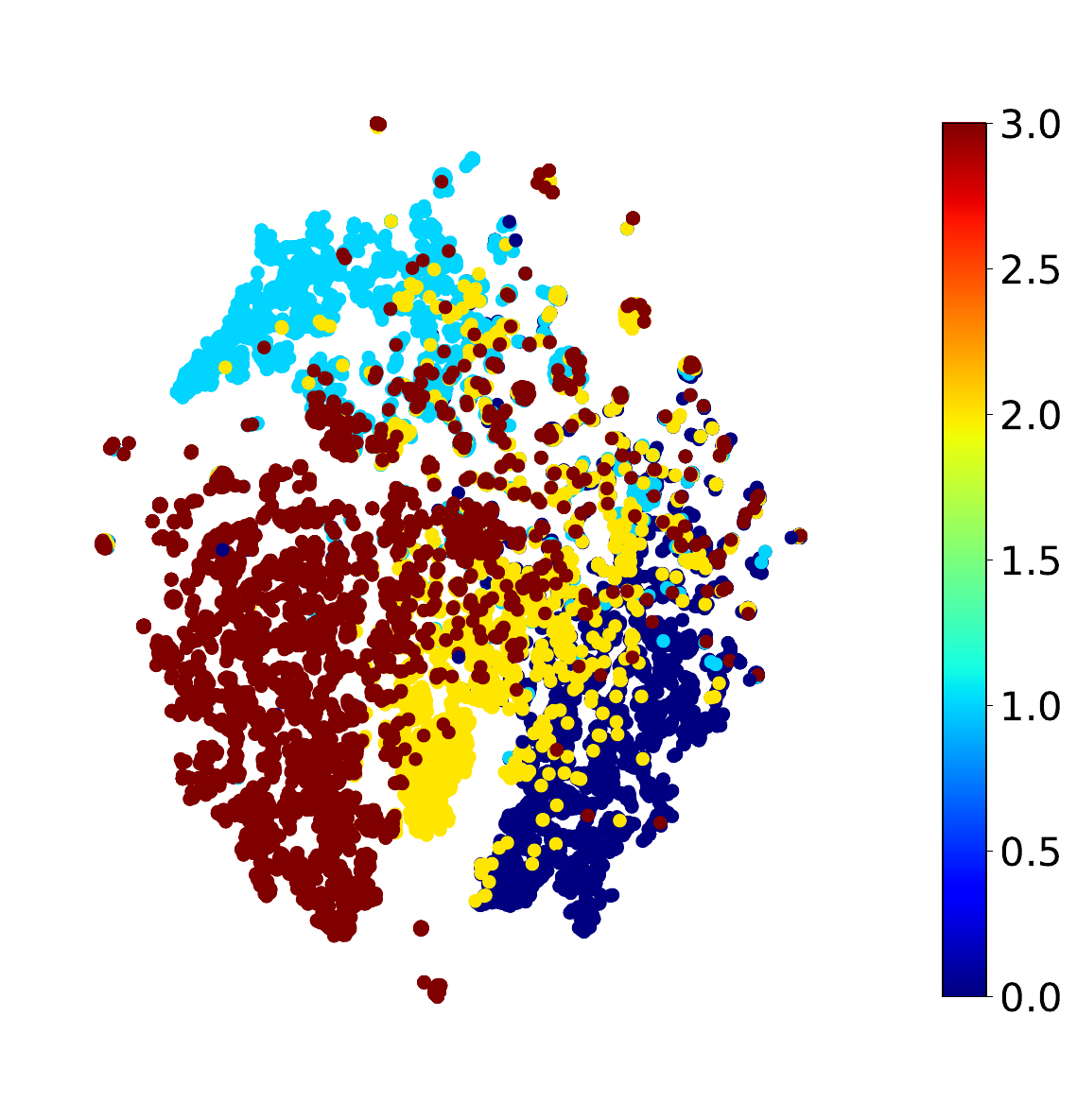}
    \end{minipage}%
    \begin{minipage}{0.20\textwidth}
        \caption{Visualization of feature embedding update on News20 dataset (3 classes) -- initial embedding (left), UniGCN embedding update (middle), DPHGNN embedding update (right).}
    \label{fig:NEWS20}
    \end{minipage}
\end{figure*}

\end{document}